\documentclass{article}

\usepackage{arxiv}

\usepackage[utf8]{inputenc}
\usepackage[T1]{fontenc}
\usepackage[round]{natbib}
\usepackage{acro}
\usepackage{algorithm}
\acsetup{format/first-long=}
\acsetup{barriers/use=true}
\acsetup{barriers/reset=true}
\acsetup{single}
\newcommand{\dac}[3]{\DeclareAcronym{#1}{short = #2, long = #3}}
\dac{auroc}{AUROC}{area under the receiver operating characteristic curve}
\dac{cdf}{CDF}{cumulative distribution function}
\dac{cvar}{CVaR}{conditional value-at-risk}
\dac{dro}{DRO}{distributionally robust optimization}
\dac{id}{iD}{in-distribution}
\dac{ml}{ML}{machine learning}
\dac{ood}{OoD}{out-of-distribution}
\dac{var}{VaR}{value-at-risk}
\usepackage{algpseudocode}

\usepackage{amsfonts}
\usepackage{amsmath}
\usepackage{amsthm}
\usepackage{bm}
\usepackage{booktabs}
\usepackage{caption}
\PassOptionsToPackage{hyphens}{url}
\usepackage{hyperref}
\hypersetup{colorlinks=true, citecolor=blue!50!black, linkcolor=red!50!black, urlcolor=green!50!black}
\usepackage[capitalize,noabbrev]{cleveref}
\crefformat{equation}{(#2#1#3)}
\crefrangeformat{equation}{(#3#1#4--#5#2#6)}
\crefmultiformat{equation}{(#2#1#3)}{ and (#2#1#3)}
{, (#2#1#3)}{ and (#2#1#3)}
\usepackage{graphicx}
\usepackage{mathtools}
\usepackage{pifont}
\usepackage{stmaryrd}
\usepackage{subcaption}
\usepackage{url}
\usepackage{xcolor}

\newtheorem{example}{Example}[section]
\newtheorem{proposition}{Proposition}[section]

\newcommand{\bell}{\bar{\ell}}
\newcommand{\tp}{\tilde{p}}

\newcommand{\vx}{\bm{x}}
\newcommand{\vz}{\bm{z}}
\newcommand{\cvar}{\operatorname{CVaR}}
\newcommand{\var}{\operatorname{VaR}}
\newcommand{\softmax}{\operatorname{softmax}}
\newcommand{\bE}{\mathbb{E}}
\newcommand{\bR}{\mathbb{R}}
\newcommand{\cC}{\mathcal{C}}
\newcommand{\cD}{\mathcal{D}}
\newcommand{\cH}{\mathcal{H}}

\newcommand{\cQ}{\mathcal{Q}}
\newcommand{\cT}{\mathcal{T}}
\newcommand{\cV}{\mathcal{V}}
\newcommand{\cX}{\mathcal{X}}
\newcommand{\cY}{\mathcal{Y}}

\newcommand{\hh}{\hat{h}}
\newcommand{\lowp}{\underline{p}}
\newcommand{\uppp}{\overline{p}}

\DeclareMathOperator*{\argmax}{arg\,max}
\DeclareMathOperator*{\argmin}{arg\,min}
\newcommand{\onecirc}{\ding{172}}
\newcommand{\twocirc}{\ding{173}}
\newcommand{\threecirc}{\ding{174}}
\newcommand{\fourcirc}{\ding{175}}
\newcommand{\plus}{\kern0.18em\raisebox{0.2ex}{\scriptsize$+$}\kern0.18em}

\title{Credal Machine Learning for \\ Risk-Averse Decision Making}

\author{
    Timo Löhr\thanks{Equal contribution.} \\ LMU Munich, MCML \\ \texttt{timo.loehr@ifi.lmu.de}
    \And
    Paul Hofman\footnotemark[1] \\ LMU Munich, MCML \\ \texttt{paul.hofman@ifi.lmu.de}
    \AND
    Maximilian Muschalik \\ LMU Munich, MCML \\ \texttt{maximilian.muschalik@lmu.de}
    \And
    Eyke Hüllermeier \\ LMU Munich, MCML, DFKI \\ \texttt{eyke@lmu.de}
}

\renewcommand{\headeright}{A Preprint}
\renewcommand{\undertitle}{A Preprint}
\renewcommand{\shorttitle}{Credal Machine Learning for Risk-Averse Decision Making}

\hypersetup{
    pdftitle={Credal Machine Learning for Risk-Averse Decision Making},
    pdfauthor={Timo Löhr, Paul Hofman, Maximilian Muschalik, Eyke Hüllermeier},
}

\begin{document}
    \maketitle

    \begin{abstract}
        In many machine learning applications, it is necessary to guard against worst-case scenarios and predictions that could result in substantial losses. In principle, this can be achieved by training risk-averse predictive models that minimize loss functions such as \ac{cvar}, rather than relying on models that perform well on average. In practice, however, the effectiveness of this approach to risk aversion is undermined by the learner's uncertainty regarding the true loss distribution and, consequently, the true \ac{cvar}. To achieve reliable risk-aversion, we propose a method in which this (epistemic) uncertainty is represented in terms of credal sets, i.e., sets of probability distributions. More specifically, we develop an efficient yet reliable learner that produces predictions in the form of credal sets and combine it with a novel decision rule that maps each credal set to a single predictive distribution for \ac{cvar} minimization. Across classification, under distribution shift, and in reinforcement learning, our approach reliably avoids catastrophic decisions, while sacrificing little in expected performance.
    \end{abstract}


    \acresetall

    \suppressfloats[t]
    \begin{figure}[t]
        \centering
        \includegraphics[width=0.5\linewidth]{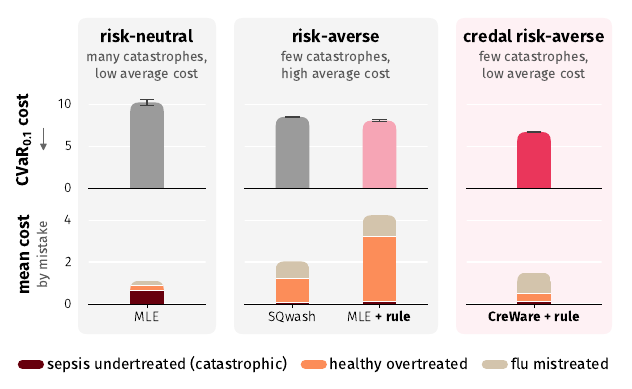}
        \caption{Risk-averse decision making. Exemplary medical setting where our method shows lower \acl{cvar} and avoids catastrophic decisions.}
        \label{fig:cost-sensitive-fig1}
    \end{figure}

    \addtocontents{toc}{\protect\setcounter{tocdepth}{0}}

    \section{Introduction}
    \Ac{ml} models deployed in safety-critical domains are typically judged by their average performance, yet an average is rarely a sufficient criterion of quality. A treatment policy in medicine, for example, must not merely be effective in expectation over patients, but also guarantee that no individual patient receives care which leads to a worst-case outcome (cf. \cref{fig:cost-sensitive-fig1}). What is needed, in other words, is a form of \textit{risk aversion}, a preference for predictions and decisions whose worst outcomes remain controlled.

    A standard way to operationalize this preference is the \ac{cvar}, which, broadly speaking, averages the losses in the upper tail of the loss distribution \citep{rockafellarOptimizationCVaR2000}. Optimal prediction under this criterion requires a transformation of the ground-truth distribution that shifts probability mass towards severe outcomes, thereby correcting for over-confidence. Yet such a transformation requires knowledge about the ground-truth distribution which is rarely available in practice. Training data may be scarce, and deployment conditions may differ from those at training time, leaving the learner epistemically uncertain about the true data generating distribution. Existing approaches to risk-averse learning do not account for such epistemic uncertainty, and may therefore fail exactly where the model knows least \citep{chowCVaROptimization2015, curiAdaptiveSampling2020}.

    Epistemic uncertainty is naturally represented by credal sets, namely sets of probability distributions which are plausible under the available evidence. Several methods for credal prediction have recently been proposed in machine learning \citep{wangCredalWrapper2025,lohrCredal2025,hofmanEfficientCredalPrediction2026}. These methods face a trade-off between expressiveness and cost. Approaches that reflect epistemic uncertainty faithfully typically rely on training an ensemble of models, whereas computationally cheaper alternatives are limited in the construction of the credal set.

    The present paper exploits the complementarity of credal prediction and \ac{cvar} minimization for risk-averse decision making: credal sets account for epistemic uncertainty about the ground truth distribution, whereas \ac{cvar} guards against worst-case outcomes once a distribution is fixed. Neither suffices alone. Trusting a precise loss distribution is dangerous when it may be misspecified, and focusing on the mean while disregarding the tail exposes a decision maker to precisely the outcomes it seeks to avoid. Building on this analysis, we develop a credal predictor that faithfully reflects epistemic uncertainty while requiring only a single forward pass, and complement it with a risk-averse decision rule that maps credal sets to single distributions for \ac{cvar} minimization.

    In summary, our \textbf{contributions} are as follows:
    \onecirc\ We establish a formal connection between credal prediction and \ac{cvar} minimization, arguing that their combination is essential for risk-averse decision making under both known and unknown risk.
    \twocirc\ We introduce a new method for credal prediction producing expressive credal sets at low computational cost.
    \threecirc\ We propose a novel risk-averse decision rule that maps a credal set to a single predictive distribution for risk-averse decision making.
    \fourcirc\ In experiments on classification, under distribution shift, and in reinforcement learning, our approach reliably avoids catastrophic decisions while sacrificing little in expected performance.

    \section{Risk-Averse ML}\label{sec:risk-averse}
    We consider an \Ac{ml} setting with supervision in the form of (training) data $\cD_{\rm{train}} = \{(\vx_i, y_i)\}_{i=1}^N \subset (\cX \times \cY)$, where $\cX \subseteq \bR^d$ is the instance space and $\cY = \{1, \ldots, K\}$ a (discrete) target space with $K$ classes. The data are assumed to be realizations of the random variables $(X_i, Y_i)$, which are independent and identically distributed (i.i.d.) according to a probability measure $P$ on $(\cX, \cY)$. We further denote by $P_X$ the marginal distribution on $\cX$. We consider a hypothesis space of probabilistic predictors $\cH = \{h : \cX \to \Delta_K\}$, where $\Delta_K$ denotes the $(K-1)$-simplex, which we will also refer to as the probability simplex. Given an instance $\vx$, the predictor $h$ thus predicts probability distributions $p(\cdot \mid \vx, h) = h(\vx)$, which are estimates of the ground-truth conditional $p^*(\cdot \mid \vx)$. For brevity, we will also write $p$ and $p^*$, respectively, whenever the conditioning is clear from the context.

    \Ac{ml} models are usually evaluated in terms of expected risk $\bE_{(X, Y) \sim P}\left[\ell(h(X), Y)\right]$, where $\ell$ is a proper scoring rule \citep{gneitingProperScoringRules2007} such as log-loss (cross-entropy) or 0/1-loss. However, when it comes to safety-critical applications such as medical decision support and autonomous driving, it is not enough to merely consider average performance. From an operational perspective, a model with a low mean error but a loss distribution with a heavy upper tail is unsafe.

    \subsection{Conditional-Value-at-Risk}
    A natural idea is to replace the expectation by a different aggregation of the loss distribution, which places more emphasis on critical outcomes. A canonical choice is the \ac{cvar} at a level $\tau \in (0,1]$, which, roughly speaking, is the mean of the upper $\tau$-tail of the loss distribution. The same functional is also known as expected shortfall \citep{acerbiExpectedShortfall2002}, average value-at-risk \citep{shapiroStochasticProgramming2014}, or superquantile \citep{laguelSuperquantiles2021}. It is defined as
    \begin{equation}\label{eq:cvar}
        \cvar_\tau(L) = \frac{1}{\tau} \int_0^\tau \var_u(L)\ du \, ,
    \end{equation}
    where $\var_u(L) = \inf_{t \in \mathbb{R}} \{t : F_L(t) \geq 1 - u\}$ is the \ac{var}, and $L = \ell(h(X), Y)$. Following \citet{rockafellarOptimizationCVaR2000}, it can also be written as
    \begin{equation}\label{eq:variational-cvar}
        \cvar_\tau(L) = \min_{v \in \bR} \left\{v + \frac{1}{\tau}\bE_P\left[(L - v)_+ \right] \right\} \, .
    \end{equation}
    Note that the minimum in \cref{eq:variational-cvar} is attained at $v^* = \var_\tau(L)$. The essential supremum is recovered by $\tau \to 0$ and the expected risk by $\tau \to 1$. The \ac{cvar} is a coherent risk measure in the sense of \citet{artznerCoherentRiskMeasures1999} and is part of a larger family of convex risk measures \citep{foellmerConvexRiskMeasures2002}, which also contains the entropic risk measure
    \citep{foellmer2011entropic}.

    The \ac{cvar} has become a common metric in risk-averse \ac{ml}, appearing in robust variants of support vector machines \citep{takedaCVaRLearning2009}, optimization of risk-averse learning objectives \citep{levyLargeScaleDRO2020,curiAdaptiveSampling2020}, reinforcement learning \citep{rigterBayesAdaptiveRL2021,greenbergRiskAverseRL2022}, and the design of fairness-aware risk measures \citep{williamsonFairnessRiskMeasures2019}. The functional in \cref{eq:cvar} is a population quantity. Empirically, it can be estimated by its plug-in on a sample of $N$ losses. Writing $L^{(1)} \geq \ldots \geq L^{(N)}$ for the ordered losses, the estimate is
    \begin{equation}\label{eq:cvar-empirical}
        \widehat{\cvar}_\tau(L) \coloneqq \frac{1}{\lceil \tau N \rceil} \sum_{i=1}^{\lceil \tau N \rceil} L^{(i)},
    \end{equation}
    that is, the average of the $\lceil \tau N \rceil$ largest losses. For $\tau = \frac{k}{N}$, this is precisely the (average) top-$k$ loss of \citet{fanLearningTopkLoss2017}. Also, statistical guarantees for learning under this criterion such as PAC-Bayesian bounds have been developed \citep{duchiLearningModelsDRO2018,mhammediPACBayesianCVaR2020,leeLearningBoundsRiskSensitive2020,somaStatisticalLearningCVaR2020}.
    \cref{app:variants} discusses variants of the \ac{cvar}.

    \subsection{Minimizing the \ac{cvar}}\label{sec:minimizing-cvar}
    The goal is thus to find a predictor that minimizes \cref{eq:cvar}. Unlike the expected loss, this problem does not decompose additively over instance-wise problems $\min_{p \in \Delta_K} \bE_{Y \sim p^*}\left[\ell(p,Y)\right]$, because \ac{cvar} is not an expectation of an instance-wise loss. Instead, it is a functional of the entire loss distribution. Fortunately, the variational form \cref{eq:variational-cvar} restores this decomposition. By interchanging the minimization over $h$ with one over $v$ such that
    \begin{align*}
        &\min_{h \in \cH} \cvar_\tau\left(\ell(h(X), Y)\right) = \\
        &\min_{v \in \bR} \left\{v + \frac{1}{\tau} \bE_{P_X} \left[\min_{p \in \Delta_K} \bE_{Y \sim p^*} \left[(\ell(p, Y) - v)_+ \right]\right]\right\},
    \end{align*}
    the objective (at fixed threshold $v$) becomes an expectation over the hinged loss $(\ell(p, Y) - v)_+$ and thus decomposes into independent instance-wise problems, which are coupled only through $v$.
    Given an instance $\vx$, this amounts to finding the distribution
    \begin{equation}\label{eq:cvar-distribution}
        \tilde{p}(\cdot \mid \vx) = \argmin_{p \in \Delta_K}\bE_{Y \sim p^*}\left[(\ell(p,Y) - v^*)_+\right].
    \end{equation}
    From \cref{eq:cvar-distribution}, the tension between the expected risk and the \ac{cvar} becomes immediately obvious.
    \begin{proposition} \label{prop:tension}
        Let $p^*\in \Delta_K$ be the ground-truth distribution for instance $\vx$ and $\ell : \Delta_K \times \cY \to \overline{\bR}$ a strictly proper scoring rule, so that, by propriety of $\ell$,
        \begin{equation*}
            p^* = \underset{p \in \Delta_K}{\argmin}\ \bE_{Y \sim p^*}[\ell(p, Y)] \, .
        \end{equation*}
        Given a threshold $v^*$, define
        \begin{equation*}
            \cT_{v^*} = \{y \in \cY : \ell(p^*, y) > v^*\} \, .
        \end{equation*}
        Then, if $(\ell, p^*, v^*)$ is tail non-degenerate, meaning that
        $0 < p^*(\cT_{v^*}) < 1$,
        and satisfies the regularity conditions stated in \cref{app:proof-minimizers},
        \begin{equation*}
            p^* \notin \underset{p \in \Delta_K}{\argmin}\  \bE_{Y \sim p^*}\left[(\ell(p,Y) - v^*)_+\right].
        \end{equation*}
    \end{proposition}
    \noindent
    In other words, the hinge in \cref{eq:cvar-distribution} destroys the propriety and causes losses below the threshold $v$ to not matter. Therefore, the minimizer shifts probability mass from likely classes to classes in the tail set $\cT_{v^*}$. This effect is strongest for highly peaked distributions, where a class with low probability is expensive. If the losses for all classes fall below the threshold, $p^*$ remains optimal.
    We provide a proof of Proposition \ref{prop:tension} in \cref{app:proof-minimizers}.

    \begin{example}[Log-loss]
Consider the binary setting, i.e., $K=2$, and let the ground-truth distribution be $p^*(\cdot \mid \vx) = (0.8, 0.2)$. For a threshold $v^* = \frac{1}{2}$,
\begin{align*}
    \underset{p \in \Delta_K}{\argmin}\ & \bE_{Y \sim p^*} \left(\ell_{\rm{log}}(p,Y) - 1/2\right)_+  = \\
    & = \left(e^{-\frac{1}{2}}, 1-e^{-\frac{1}{2}}\right) \approx \left(0.61, 0.39\right).
\end{align*}
Further examples can be found in \cref{app:examples}.
    \end{example}

    Importantly, the minimizer in \cref{eq:cvar-distribution} requires the ground-truth conditional distribution $p^*(\cdot \mid \vx)$, and (through $v^*$) the joint distribution $P$, neither of which is available in practice. Estimates are prone to error precisely in those regions where the objective is sensitive, namely in the tail of the loss distribution. This calls for caution when computing \ac{cvar}-optimal distributions.

    \section{Credal Prediction}
    \label{sec:credal}
    Committing to a single estimate of $p^*(\cdot \mid \vx)$ in \cref{eq:cvar-distribution} treats it as the truth --- ignoring the learner's \emph{epistemic} uncertainty about its predictions $p(\cdot \mid \vx)$.
    Since a single probability distribution cannot express this uncertainty, a second-order formalism is required.

    \subsection{Credal Sets from Relative Likelihood}

    Epistemic uncertainty can be represented by \emph{credal sets}, that is, (closed and convex) sets of probability distributions \citep{leviCredal1978,walleyCredal1991,augustin2014introduction}. In addition to representing partial knowledge, such sets are the natural object for decision making under ambiguity: robust Bayesian decision theory evaluates against precisely such sets of distributions \citep{gruenwaldGameTheory2004,berger2013statistical}. Given an instance $\vx$, credal sets may be induced by class-wise probability intervals \citep{camposProbabilityIntervals1994}
    \begin{equation*}
        \cQ_{\vx} = \Big\{ p : \lowp(k \mid \vx) \leq p(k \mid \vx) \leq \uppp(k \mid \vx) \Big\} \, ,
    \end{equation*}
    specified by lower and upper probabilities $\lowp$ and $\uppp$.

    The construction of $\cQ_{\vx}$ should strictly reflect the evidence in the training data, specifically in settings with little data or distribution shifts in test environments. To this end, we follow an approach that has recently been advocated by \citep{lohrCredal2025,hofmanEfficientCredalPrediction2026}. This approach is based on the notion of \emph{relative likelihood}, a well-established concept in statistical inference \citep{wassermanBelief1990}.
    With the maximum likelihood estimate $\hh = \argmax_{h \in \cH} \prod_{i=1}^N p(y_i \mid \vx_i, h)$,
    the \emph{relative likelihood} of a model $h \in \cH$ is defined as
    \begin{equation}\label{eq:lr}
        \Lambda(h) = \frac{\prod_{i=1}^N p(y_i \mid \vx_i, h)}{\prod_{i=1}^N p(y_i \mid \vx_i, \hh)} \, .
    \end{equation}
    Thus, a model is deemed plausible to the degree to which it explains the data as well as the maximum likelihood estimate, which is a classical measure of evidence in the likelihood paradigm \citep{royall2017statistical}.
    For a fixed plausibility level $\alpha \in (0,1]$, the $\alpha$-cut of the hypothesis space,
    $\mathcal{C}_\alpha = \{ h \in \mathcal{H} : \Lambda(h) \geq \alpha \}$, is known as \emph{likelihood region} in statistics.
    According to Wilks' theorem \citep{wilksLikelihood1938}, stating that, under regularity conditions, $-2$ times the log of the likelihood ratio in \cref{eq:lr}, evaluated at the ground-truth model, is asymptotically $\chi^2$-distributed, $\cC_\alpha$ can be seen as a confidence region for the ground-truth model.
    This region induces the class-wise probability bounds
    \begin{align}\label{eq:credal-bounds}
        \begin{split}
            \lowp(k \mid \vx) &= \inf_{h \in \mathcal{C}_\alpha} p(k \mid \vx, h) \\
            \uppp(k \mid \vx) &= \sup_{h \in \mathcal{C}_\alpha} p(k \mid \vx, h).
        \end{split}
    \end{align}
    The plausible model class $\mathcal{C}_\alpha$, and hence the induced probability intervals (\ref{eq:credal-bounds}), are nested in $\alpha$, shrinking from the entire probability simplex ($\alpha \downarrow 0$) to only the maximum likelihood predictions ($\alpha=1$).

    \subsection{Learning and Approximation}
    Since computing the bounds \cref{eq:credal-bounds} over the full hypothesis space is infeasible, any practical credal predictor must approximate them --- specifically in the case of neural networks, which we will focus on for the remainder of this work. Ideally, such an approximation is efficient, since the hypothesis space becomes increasingly high-dimensional with the scale of current models, and even an approximate search for plausible models grows correspondingly costly.
    It should be faithful to the evidence: the \ac{cvar} objective \cref{eq:cvar-distribution} is driven by the probabilities of the tail classes, precisely the small probabilities for which point estimates are least reliable, so the bounds should be attainable by actually plausible models. Lastly, the sets should be sensitive to changes in the shifts of the test distribution and \acl{ood} instances.

    To meet these requirements, we assess plausibility locally rather than globally. Instead of approximating $\mathcal{C}_\alpha$ over the hypothesis space $\mathcal{H}$ consisting of all neural networks with a particular architecture, we consider the \emph{local likelihood} of the test instance $\vx$. The hypotheses are then no longer models, but conditional distributions at $\vx$, that is, $\cH_{\vx} = \Delta_K$. Moreover, plausibility is judged by how well a candidate $p \in \cH_{\vx}$ explains the training data in the proximity of $\vx$, such that the resulting sets adapt, by construction, to test instances that are shifted or come from a different distribution altogether. Local likelihood estimation is a classical technique in statistics \citep{tibshiraniLocalLikelihood1987,loaderLocalLikelihood2006}, and its combination with the relative likelihood has been explored in machine learning \citep{sengeReliableClassification2014}, albeit for simpler models.

    To be efficient, we instantiate this idea using the feature embeddings of a trained network: given $h(\vx) = p(\cdot \mid \vx, h) = \softmax(g(f(\vx)))$ with feature extractor $f$ and linear head $g$, we work with the training embeddings $\vz_i = f(\vx_i)$, in which locality is cheap to evaluate. This parallels \citet{shakerAleatoricEpistemicRandomForests2020}, who assess local regions of the feature space through leaves of decision trees. In our case, these regions are considered in the learned embedding space.

    Concretely, embeddings are compared by a kernel, as is common when working with learned representations \citep{wilsonDeepKernelLearning2016}. Given a test instance $\vx$ and its embedding $\vz$, a training point with embedding $\vz_i$, receives the weight
    \begin{equation*}
        w_i(\vx) = \kappa(\vz, \vz_i) \in [0,1] \, ,
    \end{equation*}
    where $\kappa$ is a Gaussian kernel on the (whitened) embeddings. We defer details about the pre-processing and efficiency of our method to \cref{app:practical-implementation}. The weights accumulate into class-wise \emph{evidence}
    $$e_k(\vx) = \sum_{i=1}^N w_i(\vx) \llbracket y_i = k \rrbracket$$
    with total evidence $e(\vx) = \sum_{k=1}^K e_k(\vx)$. From this, the local maximum likelihood estimate is $\hat{p}(k \mid \vx) = e_k(\vx) / e(\vx)$. Note that this estimate is well-known --- it coincides with the posterior of a kernel density classifier in the embedding space \citep{parzenEstimationProbabilityDensity1962} and is a nonparametric counterpart of the Gaussian generative classifiers more recently used for feature-space uncertainty \citep{leeMahalanobisOOD2018}.
    In analogy to \citet{shakerAleatoricEpistemicRandomForests2020}, we place a multinomial likelihood around this estimate, with the evidence $e_k(\vx)$ acting as pseudo-count, which gives the \emph{local relative likelihood}
    \begin{equation*}
        \Lambda_{\vx}(p) = \prod_{k=1}^K\left(\frac{p(k \mid \vx)}{\hat{p}(k \mid \vx)}\right)^{e_k(\vx)} \, ,
    \end{equation*}
    and take the $\alpha$-cut, which admits a closed expression in the form of a KL-divergence ball around the local estimate whose radius is determined by the evidence:
    \begin{align*}
        &\left\{p \in \Delta_K :  \Lambda_{\vx}(p) \geq \alpha\right\} = \\
        & \qquad \left\{p \in \Delta_K : \operatorname{KL}(\hat{p} \parallel p) \leq \frac{\log(1/\alpha)}{e(\vx)}\right\},
    \end{align*}
    When there is little evidence near $\vx$, due to few similar training points or distribution shifts, the ball will be large --- vice versa, the more evidence, the smaller the ball. The class-wise intervals \cref{eq:credal-bounds} are obtained as the extremes of $p(k \mid \vx)$ over this ball, resulting in the credal set $\cQ_{\vx,\alpha}$. We denote our method as CreWare.

    \section{From Credal Predictions to Risk-Averse Decisions}\label{sec:risk-averse-prediction}

    Given a credal predictor, a single conditional distribution is replaced by a set $\cQ_{\vx}$, encoding the belief over plausible distributions. A decision-theoretic questions remains: how to turn a \emph{set} of distributions into a \emph{single} risk-averse decision? A natural answer is suggested by considering the \ac{cvar} itself, which, besides \cref{eq:variational-cvar}, also admits the dual representation \citep{shapiroStochasticProgramming2014}
    \begin{equation}\label{eq:cvar-dual}
        \cvar_\tau(L) = \sup \left\{\bE_S [L] : S \ll P,\ \frac{dS}{dP} \leq \frac{1}{\tau}\right\} \, .
    \end{equation}
    The \ac{cvar} is thus a worst-case expectation over reweightings of the joint distribution $P$ with bounded density ratio, and, in the sense of imprecise probability, a coherent upper prevision \citep{walleyCredal1991,pelessoniImprecisePrevisions2003}.

    Risk-averse prediction in the sense of \ac{cvar} can, therefore, be understood as ``guarding'' against every distribution in a set. The rule in \cref{eq:cvar-distribution} builds this set around the maximum likelihood estimate, but if this estimate is off, the true conditional distribution $p^*(\cdot \mid \vx)$ may lie outside of this set, so the worst case in \cref{eq:cvar-dual} offers no protection against this. The credal set fills this gap, since, by construction, it contains all distributions that are plausible given the training data. Naturally, an agent seeking to predict in a risk-averse manner should also consider the worst case over $\cQ_{\vx}$ --- applying the worst-case principle of \cref{eq:cvar-dual} to the uncertainty about the ground truth distribution itself. This motivates the following decision rule:
    \begin{equation}\label{eq:cvar-distribution-set}
        \boxed{
            \tilde{p}(\cdot \mid \vx) = \argmin_{p \in \Delta_K} \sup_{q \in \cQ_{\vx}} \bE_{Y \sim q}\left[ \big(\ell(p, Y) - v \big)_+ \right]
        }
    \end{equation}
    The resulting decision rule is $\Gamma$-maximin \citep{gilboa1989maxmin} applied to the hinged loss from \cref{eq:cvar-distribution}.

    \textbf{Special Cases.}
    The rule \cref{eq:cvar-distribution-set} contains classical decision rules as special cases of the two parameters $\tau$ and $\alpha$. For $\alpha = 1$, the credal set collapses to the maximum likelihood estimate and \cref{eq:cvar-distribution-set} reduces to the precise \ac{cvar}-optimal prediction \cref{eq:cvar-distribution} based on this estimate.
    For $\tau = 1$, the calibrated threshold falls below every attainable loss, and the hinge becomes inactive, reducing the rule to $\Gamma$-maximin of the expected loss,
    \begin{equation*}
        \tilde{p} =  \argmin_{p \in \Delta_K} \sup_{q \in \cQ_{\vx}} \bE_{Y \sim q}\left[\ell(p, Y)\right],
    \end{equation*}
    that is, the ambiguity-averse but outcome risk-neutral robust Bayes act \citep{berger2013statistical}. For losses convex in $p$ (such as the log-loss and Brier score), this becomes the distribution in the credal set maximizing the generalized entropy of \citet{gruenwaldGameTheory2004}:
    \begin{equation}\label{eq:outcome-minimax}
        \tilde{p} = \argmax_{q \in \cQ_{\vx}} \bE_{Y \sim q}\left[\ell(q, Y)\right]
    \end{equation}
    For $\tau \to 0$, the \ac{cvar} approaches the essential supremum of the loss, the threshold $v$ approaches, roughly speaking, the smallest value at which the hinge is inactive for each prediction, and the rule tends to Wald's minimax act \citep{waldStasticalDecisionFunctions1950} over the outcomes deemed plausible by the credal set:
    \begin{equation*}
        \tilde{p} =  \argmin_{p \in \Delta_K}\ \max_{y : \uppp(y \mid \vx) > 0}\ \ell(p, y) \, ,
    \end{equation*}
    which is the most pessimistic rule at the outcome level. For $(\tau, \alpha) = (1,1)$, the Bayes act, the maximum likelihood prediction, is recovered. For $\alpha \downarrow 0$, the credal set is the whole simplex, and the rule reduces to Wald's minimax act over all outcomes, independently of $\tau$. In this sense, \cref{eq:cvar-distribution-set} is a two-parameter family interpolating between expected-risk minimization, the precise \ac{cvar} prediction, robust Bayes, and Wald's minimax.

    \textbf{Calibrating the threshold.} The decision rule \cref{eq:cvar-distribution-set} presupposes $v$, but says nothing about its value. As discussed in \cref{sec:minimizing-cvar}, the threshold is a single population-level scalar, rather than a per-instance quantity, whose role is twofold. It parameterizes the decision rule, determining which losses matter, and it is at the same time the variational variable of the \ac{cvar} itself, which is optimized exactly at the \ac{var} of the loss distribution. Moreover, the roles are coupled: every threshold $v$ induces its own predictions, which we denote by $\tp_v(\cdot\mid\vx)$ via \cref{eq:cvar-distribution-set}, these predictions induce a loss distribution, and the threshold in turn sits at the $(1-\tau)$-quantile of this distribution. The predictions and the threshold are, therefore, determined jointly --- for a fixed $v$, the instance-wise problems are solved by the decision rule, so what remains is a one-dimensional problem in $v$. Practically, we solve this using a hold-out dataset $\cD_{\rm{hold}} = \{(\vx_i, y_i)\}_{i=1}^M$ as
    \begin{equation*}
        \hat{v} = \argmin_{v \in \cV} \left\{v + \frac{1}{\tau M} \sum_{i=1}^M \big(\ell(\tp_v(\cdot \mid \vx_i), y_i) - v \big)_+ \right\},
    \end{equation*}
    where $\cV \subset \bR$ is a finite grid of candidate thresholds. For every candidate threshold, the instance-wise problems \cref{eq:cvar-distribution-set} are solved on $\cD_{\rm{hold}}$. \cref{alg:threshold} provides the full procedure while we refer to \cref{app:practical-implementation} for additional practical implementation details.

    \begin{algorithm}[t]
\caption{Computation of $v$ on hold-out data.}
\label{alg:threshold}
\begin{algorithmic}
    \Require credal predictor $H$, hold-out data $\cD_{\rm{hold}}$, \\ level $\tau$, grid $\cV$
    \Ensure calibrated threshold $\hat{v}$
    \For{$i=1,\ldots,M$}
        \State $\cQ_{\vx_i} \gets H(\vx_i)$
    \EndFor
    \For{$v \in \cV$}
        \For{$i=1,\ldots,M$}
            \State $\tp_v(\cdot \mid \vx_i) \gets \underset{p \in \Delta_K}{\argmin} \underset{q \in \cQ_{\vx_i}}{\sup} \underset{Y \sim q}{\bE}\big[(\ell(p, Y) - v)_+\big]$
            \State $L_i \gets \ell(\tp_v(\cdot \mid \vx_i), y_i)$
        \EndFor
        \State $C(v) \gets v + \frac{1}{\tau M} \sum_{i=1}^M (L_i - v)_+$
    \EndFor
    \State $\hat{v} \gets \argmin_{v \in \cV} C(v)$
    \State \textbf{diagnostic:} verify $\frac{1}{M} \sum_{i=1}^M \llbracket L_i > \hat{v} \rrbracket \approx \tau$
\end{algorithmic}
    \end{algorithm}

    \section{Empirical Results}\label{sec:empirical-results}
    We evaluate the proposed credal predictor and decision rule in several environments. First, we assess the credal predictor in isolation, showing that it adapts to distribution shift and performs well in \ac{ood} detection. Second, we study the complete pipeline in risk-averse decision making under distribution shift, where it reduces \ac{cvar} relative to the baselines, with gains that grow as the shift severity increases. Third, we establish its applicability to offline reinforcement learning.

    We consider credal ensembling [CreEns] \citep{nguyen2025credal}, credal Bayesian deep learning [CreBNN] \citep{caprio2023credal}, credal wrapper [CreWra] \citep{wangCredalWrapper2025}, credal distributionally robust optimization [CreDro] \citep{wangDro2026}, credal relative likelihood [CreRL] \citep{lohrCredal2025}, and efficient credal prediction [EffCre] \citep{hofmanEfficientCredalPrediction2026} as credal baselines and SQwash \citep{laguelSuperquantileBasedLearning2022} and AdaCVaR \citep{curiAdaptiveSampling2020} as risk-averse baselines. We publish the code for all experiments in GitHub\footnote{\url{https://github.com/timoverse/credal-ml-risk-averse-dm}} while all experimental details are described in \cref{app:experimental}.

    \subsection{Credal Sets under Distribution Shift}\label{sec:shift-set-size}
    In order to assess the quality of the credal sets generated by our novel approach, we study them under distribution shifts. Such shifts typically degrade predictive (probabilistic) performance, and for a risk-averse system, this degradation should be noticed. As the incoming (test) data shifts away from the training distribution, the predictor becomes increasingly less able to make accurate probabilistic predictions, and a proper credal predictor should represent this by increasing epistemic uncertainty, i.e., by enlarging the sets.

    We use a ResNet18 \citep{heResNet2016} trained on the BloodMNIST dataset from the MedMNIST collection \citep{yangMedMNIST2021} and use the corruptions provided by MedMNIST-C \citep{diSalvoMedMNISTC2024} at test time. We measure the set size of the credal set $\cQ_{\vx}$ by
    \begin{equation}\label{eq:set-size}
        \frac{1}{K}\sum_{k=1}^K \uppp(k \mid \vx) - \lowp(k \mid \vx) \, ,
    \end{equation}
    and average over all corruptions of a given severity. We compute the accuracy at each severity as a reference.

    \cref{fig:distribution-shift} shows the results. The baseline credal predictors increase their set size only slightly across severities, while the accuracy decreases significantly. Our method, in contrast, generates substantially larger sets as the severity increases, and is therefore more aligned with the decreasing accuracy. We provide results with additional datasets in \cref{app:distributional-shifts}.

    \begin{figure}[t]
        \centering
        \includegraphics[width=0.5\linewidth]{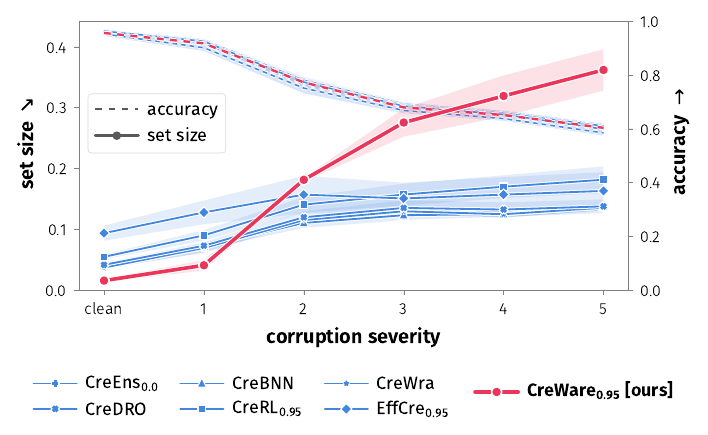}
        \caption{Distribution shift on BloodMNIST. Set size and accuracy against shift severity, mean and standard deviation over three seeds. CreWare's sets grow as accuracy declines, while the baselines' barely grow.}
        \label{fig:distribution-shift}
    \end{figure}

    \subsection{Out-of-Distribution Detection}

    Beyond predicting in a risk-averse manner on \ac{id} data from the training distribution or shifted versions thereof, a credal predictor should also recognize when predictions can not be trusted at all. That is, when it is faced with \ac{ood} data, the credal sets should expand and the epistemic uncertainty should increase accordingly, separating \ac{id} from \ac{ood} inputs, a task commonly used to assess the quality of credal sets \citep{wangCredalWrapper2025,wangCredalDistillation2026,lohrCredal2025}.
    We use the epistemic uncertainty measure
    \begin{equation}\label{eq:epistemic-uncertainty}
        \operatorname{EU}(\cQ_{\vx}) =  \overline{\operatorname{S}}(\cQ_{\vx}) - \underline{\operatorname{S}}(\cQ_{\vx}) \, ,
    \end{equation}
    where $\overline{\operatorname{S}}(\cQ_{\vx}) = \sup_{p(\cdot \mid \vx) \in \cQ_{\vx}} \operatorname{S}(p(\cdot \mid \vx))$ with Shannon entropy $\operatorname{S}$ \citep{abellanUncertaintyMeasure2006}. \ac{ood} detection is then binary classification between \ac{id} and \ac{ood} instances based on \cref{eq:epistemic-uncertainty}, and evaluated by the \ac{auroc}.

    We follow the setup from the OpenOOD benchmark \citep{zhangOpenOOD2024} --- using a ResNet18 --- with CIFAR-10 \citep{krizhevsky2009cifar} as the \ac{id} dataset, while CIFAR-100 \citep{krizhevsky2009cifar}, and TinyImageNet \citep{le2015tiny} are near-\ac{ood} datasets and SVHN \citep{netzer2011svhn}, and Places365 \citep{zhouPlaces3652018} are far-\ac{ood} datasets.

    \begin{table}[t]
        \centering
        \caption{\Acl*{ood} detection with CIFAR-10 as \acs*{id} data. \acs*{auroc}, mean and standard deviation over three runs, best in \textbf{bold}. CreWare (ours) outperforms all baselines on near- and far-\acs*{ood} data.}
        \label{tab:ood}
        \setlength{\tabcolsep}{4pt}
        \resizebox{0.6\linewidth}{!}{%
            \begin{tabular}{@{}lcc@{\hskip 5pt}|@{\hskip 5pt}cc@{}}
                \toprule
                & \multicolumn{2}{c@{\hskip 5pt}}{\textbf{Near-OOD}} & \multicolumn{2}{c}{\textbf{Far-OOD}} \\
                \textbf{Method}         & \textbf{CIFAR-100}                    & \textbf{T.ImageNet}                   & \textbf{SVHN}                         & \textbf{Places365}                    \\ \midrule
                $\text{CreEns}_{0.0}$   & $.917 \scriptstyle \pm .000$          & $.911 \scriptstyle \pm .001$          & $.970 \scriptstyle \pm .003$ & $.919 \scriptstyle \pm .003$ \\
                CreBNN                  & $.867 \scriptstyle \pm .007$          & $.866 \scriptstyle \pm .006$          & $.892 \scriptstyle \pm .016$ & $.872 \scriptstyle \pm .006$ \\
                CreWra                  & $.918 \scriptstyle \pm .000$          & $.912 \scriptstyle \pm .001$          & $.971 \scriptstyle \pm .003$ & $.920 \scriptstyle \pm .003$ \\
                CreDro                  & $.918 \scriptstyle \pm .001$          & $.911 \scriptstyle \pm .001$          & $.972 \scriptstyle \pm .001$ & $.920 \scriptstyle \pm .003$ \\
                $\text{CreRL}_{0.95}$   & $.913 \scriptstyle \pm .003$          & $.908 \scriptstyle \pm .001$          & $.969 \scriptstyle \pm .007$ & $.917 \scriptstyle \pm .003$ \\
                $\text{EffCre}_{0.95}$  & $.859 \scriptstyle \pm .001$          & $.857 \scriptstyle \pm .001$ & $.911 \scriptstyle \pm .006$ & $.871 \scriptstyle \pm .003$ \\
                \midrule
                $\text{CreWare}_{0.95}$ & $\mathbf{.931 \scriptstyle \pm .002}$ & $\mathbf{.929 \scriptstyle \pm .001}$ & $\mathbf{.980 \scriptstyle \pm .005}$ & $\mathbf{.941 \scriptstyle \pm .002}$ \\
                \bottomrule
            \end{tabular}%
        }
    \end{table}
    \cref{tab:ood} shows the results. Our method achieves the highest \ac{auroc} in near- and far-\ac{ood} settings, demonstrating that the credal sets adapt suitably to \ac{ood} instances.
    We provide results for additional datasets in \cref{app:ood-results}. In addition, we perform \ac{ood} detection with different measures of uncertainty in \cref{app:ood-measures}, as \cref{eq:epistemic-uncertainty} has recently been criticized for violating desirable properties \citep{huellermeierCredalUncertainty2022}.

    While the evaluation here is standard, it deserves a caveat. Recent work has argued that supervised \ac{ood} detection is inherently misspecified, because the (epistemic) uncertainty of Bayesian and ensemble methods concerns the labels rather than whether an instance stems from another distribution \citep{liOODWrongQuestions2025}. We do not claim to resolve this issue, yet, unlike ensemble-based credal sets, the local relative likelihood explicitly compares a test instance with the training data. Thus, our credal sets widen exactly when the feature embeddings disagree with the training embeddings.
    Complementing the \ac{ood} results, we assess the quality of credal sets by their coverage of the ground-truth distribution and their efficiency in \cref{app:coverage-efficiency}.

    \subsection{Risk-Aversion under Distribution Shift}
    The previous section assessed the credal sets in isolation, showing their adaptability to distributions shifts and \ac{ood} instances. We now focus on the complete pipeline, which combines the novel credal predictor and the proposed \ac{cvar} decision rule. Risk-aversion is particularly import in medical settings, because a single catastrophic decision may outweigh many good ones and, importantly, distribution shifts are not uncommon \citep{kohWilds2021,finlaysonDatasetShift2021}.

    We therefore train a ResNet18 on PathMNIST from the MedMNIST collection \citep{yangMedMNIST2021} and evaluate the complete pipeline on the clean test split as well as on its corrupted versions from MedMNIST-C at five levels of severity. Note that not even the clean split is free of shift as it stems from different patients \citep{yangMedMNIST2021}. For each severity, we compute the empirical \ac{cvar} \cref{eq:cvar-empirical} of the test losses at $\tau \in [0.01, 0.05, 0.2]$ with the threshold of our decision rule calibrated on hold-out data based on \cref{alg:threshold}. Since risk-averse prediction inherits an intrinsic trade-off between tail and average performance, we also report the mean loss.

    The results are shown in \cref{fig:risk-averse-shift}. Our method (CreWare\plus rule) outperforms or equals the risk-averse baselines and the CreWra method (equipped with our decision rule), at each level $\tau$, and across corruption severities in terms of \ac{cvar}. At the same time, it attains a better mean loss at high corruption severities, while sacrificing little at low corruption severities.
    \cref{app:res-risk-aversion-shift} provides results with an additional dataset, the Brier loss, and an ablation for $\alpha$ and $\tau$.
    \begin{figure}
        \centering
        \includegraphics[width=0.5\linewidth]{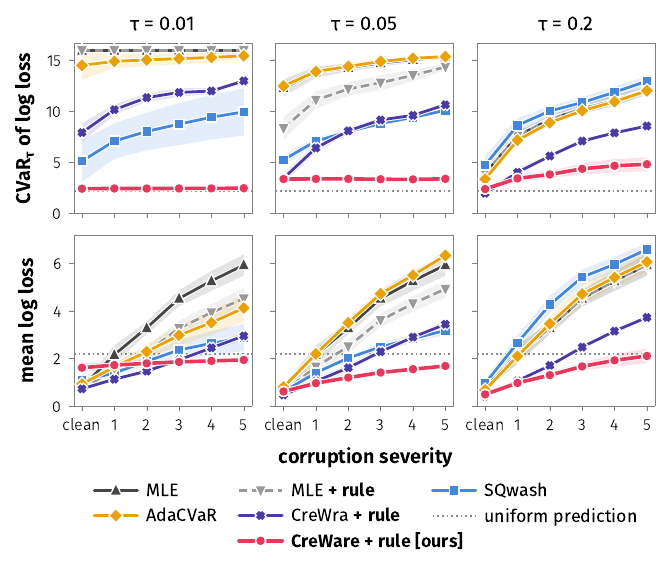}
        \caption{Risk-averse prediction under distribution shift with PathMNIST. Our method beats baselines in terms of CVaR, while sacrificing little in mean loss.}
        \label{fig:risk-averse-shift}
    \end{figure}

    \begin{figure}[t]
        \centering
        \includegraphics[width=0.55\linewidth]{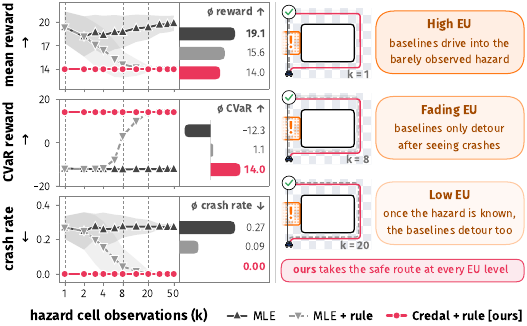}
        \caption{Offline driving task. \textbf{Left}: mean reward, CVaR reward, and crash rate vs.\ hazard cell observations $k$. \textbf{Right}: routes at three values of $k$ (decreasing epistemic uncertainty). Credal\plus rule (ours) acts safe, whereas the MLE always drives into the hazard and MLE\plus rule does so while the hazard is under-observed.}
        \label{fig:driving}
    \end{figure}

    \subsection{Risk-Averse Reinforcement Learning}
    \label{sec:rl-driving}
    We instantiate our pipeline in a reinforcement learning setting, where the loss is a negated task reward rather than a proper scoring rule. The CVaR construction of \cref{eq:cvar-empirical,eq:cvar-distribution-set} is agnostic to this distinction.
    In this setup, an agent aims to drive a car safely to a target cell (\cref{fig:driving}). A short route crosses three hazard cells, beats the cost of a longer, safe detour in expectation, but crashes in $\approx 27\%$ of the drives. The agent plans offline from a fixed dataset. Each non-hazard cell is well sampled while a hazard cell is observed $k$ times, hence epistemic uncertainty about the hazard persists for small $k$. Per cell $s$, the observed counts yield an MLE $\hat p_s$ of the crash probability, and a relative-likelihood $\alpha$-cut yields the credal set $\cQ_s=[\lowp_s,\uppp_s]$.

    We compare three agents. MLE maximizes the expected reward under the point estimates. Credal\plus rule (ours) applies the decision rule \cref{eq:cvar-distribution-set} to the per-cell credal sets minimizing \ac{cvar} over the negated reward. The MLE\plus rule serves as an ablation. It replaces the credal set by the MLE prediction and thus optimizes, under the estimated model, the \ac{cvar} objective pursued by risk-averse reinforcement learning methods such as CeSoR \citep{greenbergRiskAverseRL2022}.

    The point-estimate risk-averse agent, MLE\plus rule, is unsafe exactly while the hazard is under-observed (see \cref{fig:driving}). As long as no crash has been recorded, $\hat p_s = 0$, it perceives no tail and drives into it. Credal\plus rule (ours), in contrast, maintains a $0\%$ crash rate across the entire range. Before any crash is observed, the credal set represents the epistemic uncertainty about hazard cells, thus avoiding the short route. Once the crash probability is known, the \ac{cvar} objective rejects the shortcut for its high risk. The comparison with MLE\plus rule isolates the contribution of the credal set, since under the same risk objective only the credal agent acts safe throughout. Results on a larger map show the same picture (\cref{app:rl-results}).

    \section{Additional Related Work}
    Credal sets originated in the imprecise probability literature \citep{leviCredal1978,walleyCredal1991,augustin2014introduction}, and have entered \ac{ml} to represent epistemic uncertainty \citep{coraniBnn2012,huellermeierUncertainty2021}. Existing credal predictors, whether built from ensembles \citep{wangCredalWrapper2025,nguyen2025credal}, sets of priors \citep{caprio2023credal}, or relative likelihood \citep{lohrCredal2025}, derive the set from disagreement among models trained on the same data, so that its width is governed by the training procedure rather than the test instance. Our construction, in contrast, fits the set post hoc, similar to \citet{hofmanEfficientCredalPrediction2026}, but does so locally for each instance. Decision rules acting on credal sets, such as $\Gamma$-maximin \citep{gilboa1989maxmin}, maximality \citep{walleyCredal1991}, and E-admissibility \citep{leviCredal1978}, hedge against ignorance about the ground-truth distribution, yet aggregate outcomes by expectation over the population, and thus may accept a known but severe risk.

    The aleatoric perspective is the subject of a largely separate literature, in which the \ac{cvar} is minimized directly \citep{fanLearningTopkLoss2017,curiAdaptiveSampling2020} or as an instance of \ac{dro} \citep{duchiLearningModelsDRO2018,kuhnDRO2025}. The ambiguity set of the latter reweights the empirical distribution rather than ranging over models, and is in this sense orthogonal to the credal set. Keeping the two apart is what allows the outcome tail and the credal set to be governed by distinct parameters $\tau$ and $\alpha$. The same distinction recurs in sequential settings (e.g.\ reinforcement learning), with risk-averse methods targeting the tail of the return \citep{chowCVaROptimization2015,greenbergRiskAverseRL2022}, and robust methods optimizing against a set of plausible models \citep{iyengar2005robust,wiesemann2013robust}. Safety under epistemic uncertainty thus requires both. Closest to our work in RL is \citet{rigterRLEpistemic2023}, which addresses both sources, though through a second-order posterior over models rather than the prior-free credal sets used here. \cref{app:related-work} gives an extended discussion of related work.

    \section{Conclusion}
    We have introduced a novel approach to risk-averse prediction in settings of high epistemic uncertainty, arguing that such settings require a combination of both, risk-aversion and reliability. Indeed, risk-averse decisions are not very useful when they are based on incorrect probability predictions, which, however, is likely if the learner faces high epistemic uncertainty. To incentivize risk-aversion, we proposed to use the \ac{cvar} as a loss function\,---\,a challenge for learning, because this loss does not decompose over training instances. Moreover, to ensure reliability, we proposed a credal approach and developed a new credal predictor based on the local relative likelihood. Finally, combining these two components, we proposed a risk-averse decision rule to minimize the \ac{cvar} over a credal set.

    Our approach comes with limitations which point to interesting directions for future work. Since the construction relies on the quality of the embeddings, its sensitivity deserves closer investigation. Another open question is whether the credal sets allow for formal guarantees similar to conformal prediction.

    \section*{Acknowledgements}

    \textit{Timo Löhr}, \textit{Maximilian Muschalik} and \textit{Eyke Hüllermeier} gratefully acknowledge funding by the Deutsche Forschungsgemeinschaft (DFG, German Research Foundation): TRR 318/3 2026 – 438445824.
    We acknowledge the Munich Center for Machine Learning (MCML) for their support of this work. The authors gratefully acknowledge LMU Klinikum for providing computing resources on their Clinical Open Research Engine (CORE).

    \bibliography{references}

@inproceedings{greenbergRiskAverseRL2022,
  author       = {Ido Greenberg and
                  Yinlam Chow and
                  Mohammad Ghavamzadeh and
                  Shie Mannor},
  title        = {Efficient Risk-Averse Reinforcement Learning},
  booktitle    = {Advances in Neural Information Processing Systems 35: Annual Conference on Neural Information Processing Systems},
  year         = {2022}
}

@article{gneitingProperScoringRules2007,
  title={Strictly proper scoring rules, prediction, and estimation},
  author={Gneiting, Tilmann and Raftery, Adrian E},
  journal={Journal of the American statistical Association},
  volume={102},
  number={477},
  pages={359--378},
  year={2007},
  publisher={Taylor \& Francis}
}

@article{foellmerConvexRiskMeasures2002,
  author       = {Hans F{\"{o}}llmer and
                  Alexander Schied},
  title        = {Convex measures of risk and trading constraints},
  journal      = {Finance Stochastics},
  volume       = {6},
  number       = {4},
  pages        = {429--447},
  year         = {2002},
}

@article{artznerCoherentRiskMeasures1999,
  title={Coherent measures of risk},
  author={Artzner, Philippe and Delbaen, Freddy and Eber, Jean-Marc and Heath, David},
  journal={Mathematical finance},
  volume={9},
  number={3},
  pages={203--228},
  year={1999},
  publisher={Wiley Online Library}
}

@article{rockafellarOptimizationCVaR2000,
  title={Optimization of conditional value-at-risk},
  author={Rockafellar, R Tyrrell and Uryasev, Stanislav and others},
  journal={Journal of risk},
  volume={2},
  pages={21--42},
  year={2000}
}

@article{gruenwaldGameTheory2004,
  title={Game theory, maximum entropy, minimum discrepancy and robust Bayesian decision theory},
  author={Peter Gr{\"u}nwald and A Philip Dawid},
  journal={Annals of Statistics},
  year={2004},
  volume={32},
  pages={1367-1433},
}

@inproceedings{zhaiBoostedCVaRClassification2021,
  author       = {Runtian Zhai and
                  Chen Dan and
                  Arun Sai Suggala and
                  J. Zico Kolter and
                  Pradeep Ravikumar},
  title        = {Boosted CVaR Classification},
  booktitle    = {Advances in Neural Information Processing Systems 34: Annual Conference on Neural Information Processing Systems},
  pages        = {21860--21871},
  year         = {2021},
}

@article{chenCredalAmbiguitySets2026,
  author       = {Mengqi Chen and
                  Thomas B. Berrett and
                  Theodoros Damoulas and
                  Michele Caprio},
  title        = {Bulk-Calibrated Credal Ambiguity Sets: Fast, Tractable Decision Making under Out-of-Sample Contamination},
  journal      = {CoRR},
  volume       = {abs/2601.21324},
  year         = {2026},
}

@article{kuhnDRO2025,
  author       = {Daniel Kuhn and
                  Soroosh Shafiee and
                  Wolfram Wiesemann},
  title        = {Distributionally robust optimization},
  journal      = {Acta Numer.},
  volume       = {34},
  pages        = {579--804},
  year         = {2025},
}

@inproceedings{chowCVaROptimization2015,
  author       = {Yinlam Chow and
                  Aviv Tamar and
                  Shie Mannor and
                  Marco Pavone},
  title        = {Risk-Sensitive and Robust Decision-Making: a CVaR Optimization Approach},
  booktitle    = {Advances in Neural Information Processing Systems 28: Annual Conference on Neural Information Processing Systems 2015},
  pages        = {1522--1530},
  year         = {2015},
}

@inproceedings{mhammediPACBayesianCVaR2020,
  author       = {Zakaria Mhammedi and
                  Benjamin Guedj and
                  Robert C. Williamson},
  title        = {PAC-Bayesian Bound for the Conditional Value at Risk},
  booktitle    = {Advances in Neural Information Processing Systems 33: Annual Conference on Neural Information Processing Systems},
  year         = {2020},
}

@article{sluijtermanLikelihoodRatioNNs2025,
  author       = {Laurens Sluijterman and
                  Eric Cator and
                  Tom Heskes},
  title        = {Likelihood-ratio-based confidence intervals for neural networks},
  journal      = {Mach. Learn.},
  volume       = {114},
  number       = {4},
  pages        = {116},
  year         = {2025},
}

@article{camposProbabilityIntervals1994,
  author       = {Luis M. de Campos and
                  Juan F. Huete and
                  Seraf{\'{\i}}n Moral},
  title        = {Probability Intervals: a Tool for uncertain Reasoning},
  journal      = {Int. J. Uncertain. Fuzziness Knowl. Based Syst.},
  volume       = {2},
  number       = {2},
  pages        = {167--196},
  year         = {1994},
}

@book{sutton1998reinforcement,
  title={Reinforcement learning: An introduction},
  author={Sutton, Richard S and Barto, Andrew G and others},
  volume={1},
  number={1},
  year={1998},
  publisher={MIT press Cambridge}
}

@article{bauerle2011markov,
  title={Markov decision processes with average-value-at-risk criteria},
  author={B{\"a}uerle, Nicole and Ott, Jonathan},
  journal={Mathematical Methods of Operations Research},
  volume={74},
  number={3},
  pages={361--379},
  year={2011},
  publisher={Springer}
}

@inproceedings{shakerAleatoricEpistemicRandomForests2020,
  author       = {Mohammad Hossein Shaker and
                  Eyke H{\"{u}}llermeier},
  title        = {Aleatoric and Epistemic Uncertainty with Random Forests},
  booktitle    = {Advances in Intelligent Data Analysis {XVIII} - 18th International
                  Symposium on Intelligent Data Analysis},
  series       = {Lecture Notes in Computer Science},
  volume       = {12080},
  pages        = {444--456},
  publisher    = {Springer},
  year         = {2020},
}

@inproceedings{huellermeierCredalUncertainty2022,
  author       = {Eyke H{\"{u}}llermeier and
                  S{\'{e}}bastien Destercke and
                  Mohammad Hossein Shaker},
  title        = {Quantification of Credal Uncertainty in Machine Learning: {A} Critical
                  Analysis and Empirical Comparison},
  booktitle    = {Uncertainty in Artificial Intelligence, Proceedings of the Thirty-Eighth
                  Conference on Uncertainty in Artificial Intelligence},
  series       = {Proceedings of Machine Learning Research},
  volume       = {180},
  pages        = {548--557},
  publisher    = {{PMLR}},
  year         = {2022},
}

@article{laguelSuperquantileBasedLearning2022,
  author       = {Yassine Laguel and
                  J{\'{e}}r{\^{o}}me Malick and
                  Za{\"{\i}}d Harchaoui},
  title        = {Superquantile-Based Learning: {A} Direct Approach Using Gradient-Based Optimization},
  journal      = {J. Signal Process. Syst.},
  volume       = {94},
  number       = {2},
  pages        = {161--177},
  year         = {2022},
}

@article{sagawaDROGroupShifts2019,
  author       = {Shiori Sagawa and
                  Pang Wei Koh and
                  Tatsunori B. Hashimoto and
                  Percy Liang},
  title        = {Distributionally Robust Neural Networks for Group Shifts: On the Importance of Regularization for Worst-Case Generalization},
  journal      = {CoRR},
  volume       = {abs/1911.08731},
  year         = {2019},
}

@inproceedings{curiAdaptiveSampling2020,
  author       = {Sebastian Curi and
                  Kfir Y. Levy and
                  Stefanie Jegelka and
                  Andreas Krause},
  title        = {Adaptive Sampling for Stochastic Risk-Averse Learning},
  booktitle    = {Advances in Neural Information Processing Systems 33: Annual Conference on Neural Information Processing Systems},
  year         = {2020},
}

@article{niPossibilisticPredictiveUncertainty2026,
  author       = {Yao Ni and
                  Jeremie Houssineau and
                  Yew Soon Ong and
                  Piotr Koniusz},
  title        = {Possibilistic Predictive Uncertainty for Deep Learning},
  journal      = {CoRR},
  volume       = {abs/2605.00600},
  year         = {2026},
}

@book{shapiroStochasticProgramming2014,
  author       = {Alexander Shapiro and
                  Darinka Dentcheva and
                  Andrzej Ruszczynski},
  title        = {Lectures on Stochastic Programming - Modeling and Theory, Second Edition},
  series       = {{MOS-SIAM} Series on Optimization},
  volume       = {16},
  publisher    = {{SIAM}},
  year         = {2014},
}

@article{pattonDynamicModelsES2019,
title = {Dynamic semiparametric models for expected shortfall (and Value-at-Risk)},
journal = {Journal of Econometrics},
volume = {211},
number = {2},
pages = {388-413},
year = {2019},
author = {Andrew J. Patton and Johanna F. Ziegel and Rui Chen},
}

@article{dimitriadisExpectedShortfall2019,
author = {Timo Dimitriadis and Sebastian Bayer},
title = {{A joint quantile and expected shortfall regression framework}},
volume = {13},
journal = {Electronic Journal of Statistics},
number = {1},
publisher = {Institute of Mathematical Statistics and Bernoulli Society},
pages = {1823 -- 1871},
year = {2019},
}

@article{somaStatisticalLearningCVaR2020,
  author       = {Tasuku Soma and
                  Yuichi Yoshida},
  title        = {Statistical Learning with Conditional Value at Risk},
  journal      = {CoRR},
  volume       = {abs/2002.05826},
  year         = {2020},
}

@article{duchiLearningModelsDRO2018,
  author       = {John C. Duchi and
                  Hongseok Namkoong},
  title        = {Learning Models with Uniform Performance via Distributionally Robust
                  Optimization},
  journal      = {CoRR},
  volume       = {abs/1810.08750},
  year         = {2018},
}

@inproceedings{leeLearningBoundsRiskSensitive2020,
  author       = {Jaeho Lee and
                  Sejun Park and
                  Jinwoo Shin},
  title        = {Learning Bounds for Risk-sensitive Learning},
  booktitle    = {Advances in Neural Information Processing Systems 33: Annual Conference on Neural Information Processing Systems},
  year         = {2020},
}

@inproceedings{williamsonFairnessRiskMeasures2019,
  author       = {Robert C. Williamson and
                  Aditya Krishna Menon},
  title        = {Fairness risk measures},
  booktitle    = {Proceedings of the 36th International Conference on Machine Learning},
  series       = {Proceedings of Machine Learning Research},
  volume       = {97},
  pages        = {6786--6797},
  publisher    = {{PMLR}},
  year         = {2019},
}

@inproceedings{rigterBayesAdaptiveRL2021,
  author       = {Marc Rigter and
                  Bruno Lacerda and
                  Nick Hawes},
  title        = {Risk-Averse Bayes-Adaptive Reinforcement Learning},
  booktitle    = {Advances in Neural Information Processing Systems 34: Annual Conference on Neural Information Processing Systems},
  pages        = {1142--1154},
  year         = {2021},
}

@inproceedings{levyLargeScaleDRO2020,
  author       = {Daniel Levy and
                  Yair Carmon and
                  John C. Duchi and
                  Aaron Sidford},
  title        = {Large-Scale Methods for Distributionally Robust Optimization},
  booktitle    = {Advances in Neural Information Processing Systems 33: Annual Conference on Neural Information Processing Systems},
  year         = {2020},
}

@article{takedaCVaRLearning2009,
  author       = {Akiko Takeda and
                  Takafumi Kanamori},
  title        = {A robust approach based on conditional value-at-risk measure to statistical learning problems},
  journal      = {Eur. J. Oper. Res.},
  volume       = {198},
  number       = {1},
  pages        = {287--296},
  year         = {2009},
}

@article{laguelSuperquantiles2021,
  title={Superquantiles at work: Machine learning applications and efficient subgradient computation},
  author={Laguel, Yassine and Pillutla, Krishna and Malick, J{\'e}r{\^o}me and Harchaoui, Zaid},
  journal={Set-Valued and Variational Analysis},
  volume={29},
  number={4},
  pages={967--996},
  year={2021},
  publisher={Springer}
}

@article{acerbiExpectedShortfall2002,
  title={On the coherence of expected shortfall},
  author={Acerbi, Carlo and Tasche, Dirk},
  journal={Journal of banking \& finance},
  volume={26},
  number={7},
  pages={1487--1503},
  year={2002},
  publisher={Elsevier}
}

@inproceedings{fanLearningTopkLoss2017,
  author       = {Yanbo Fan and
                  Siwei Lyu and
                  Yiming Ying and
                  Bao{-}Gang Hu},
  title        = {Learning with Average Top-k Loss},
  booktitle    = {Advances in Neural Information Processing Systems 30: Annual Conference on Neural Information Processing Systems},
  pages        = {497--505},
  year         = {2017},
}

@article{diSalvoMedMNISTC2024,
  author       = {Francesco Di Salvo and
                  Sebastian Doerrich and
                  Christian Ledig},
  title        = {MedMNIST-C: Comprehensive benchmark and improved classifier robustness by simulating realistic image corruptions},
  journal      = {CoRR},
  volume       = {abs/2406.17536},
  year         = {2024},
}

@book{walleyCredal1991,
	author = {Peter Walley},
	editor = {},
	publisher = {Chapman \& Hall},
	title = {Statistical Reasoning with Imprecise Probabilities},
	year = {1991}
}

@Inbook{leviCredal1978,
author="Levi, Isaac",
title="On Indeterminate Probabilities",
bookTitle="Foundations and Applications of Decision Theory: Volume I Theoretical Foundations",
year="1978",
publisher="Springer",
address="Dordrecht",
pages="233--261",
}

@book{augustin2014introduction,
  title={Introduction to imprecise probabilities},
  author={Augustin, Thomas and Coolen, Frank PA and De Cooman, Gert and Troffaes, Matthias CM},
  year={2014},
  publisher={John Wiley \& Sons}
}

@article{huellermeierUncertainty2021,
          volume = {110},
           pages = {457--506},
          author = {Eyke H\"ullermeier and Willem Waegeman},
           title = {Aleatoric and epistemic uncertainty in machine learning: an introduction to concepts and methods},
          number = {3},
         journal = {Machine Learning},
            year = {2021},
}

@Inbook{coraniBnn2012,
author="Corani, G.
and Antonucci, A.
and Zaffalon, M.",
title="Bayesian Networks with Imprecise Probabilities: Theory and Application to Classification",
bookTitle="Data Mining: Foundations and Intelligent Paradigms: Volume 1: Clustering, Association and Classification",
year="2012",
publisher="Springer",
pages="49--93",
}

@article{wilksLikelihood1938,
    author = "Wilks, S. S.",
    title = "{The Large-Sample Distribution of the Likelihood Ratio for Testing Composite Hypotheses}",
    journal = "Annals Math. Statist.",
    volume = "9",
    number = "1",
    pages = "60--62",
    year = "1938"
}

@book{shaferLikelihood1976,
 author = {Glenn Shafer},
 publisher = {Princeton University Press},
 title = {A Mathematical Theory of Evidence},
 urldate = {2026-07-14},
 year = {1976}
}

@article{lecunMNIST1998,
  author       = {Yann LeCun and
                  L{\'{e}}on Bottou and
                  Yoshua Bengio and
                  Patrick Haffner},
  title        = {Gradient-based learning applied to document recognition},
  journal      = {Proc. {IEEE}},
  volume       = {86},
  number       = {11},
  pages        = {2278--2324},
  year         = {1998},
}

@inproceedings{cimpoiTextures2014,
  author       = {Mircea Cimpoi and
                  Subhransu Maji and
                  Iasonas Kokkinos and
                  Sammy Mohamed and
                  Andrea Vedaldi},
  title        = {Describing Textures in the Wild},
  booktitle    = {2014 {IEEE} Conference on Computer Vision and Pattern Recognition},
  pages        = {3606--3613},
  publisher    = {{IEEE} Computer Society},
  year         = {2014},
}

@article{zhouPlaces3652018,
  author       = {Bolei Zhou and
                  {\`{A}}gata Lapedriza and
                  Aditya Khosla and
                  Aude Oliva and
                  Antonio Torralba},
  title        = {Places: {A} 10 Million Image Database for Scene Recognition},
  journal      = {{IEEE} Trans. Pattern Anal. Mach. Intell.},
  volume       = {40},
  number       = {6},
  pages        = {1452--1464},
  year         = {2018},
}

@techreport{le2015tiny,
  author      = {Ya Le and Xuan Yang},
  title       = {Tiny ImageNet Visual Recognition Challenge},
  institution = {Stanford University},
  type        = {{CS} 231N course report},
  year        = {2015}
}

@inproceedings{netzer2011svhn,
  author    = {Yuval Netzer and Tao Wang and Adam Coates and
               Alessandro Bissacco and Bo Wu and Andrew Y. Ng},
  title     = {Reading Digits in Natural Images with Unsupervised Feature Learning},
  booktitle = {{NIPS} Workshop on Deep Learning and Unsupervised Feature Learning},
  year      = {2011},
}

@techreport{krizhevsky2009cifar,
  author      = {Alex Krizhevsky},
  title       = {Learning Multiple Layers of Features from Tiny Images},
  institution = {University of Toronto},
  year        = {2009},
}

@article{pelessoniImprecisePrevisions2003,
  author       = {Renato Pelessoni and
                  Paolo Vicig},
  title        = {Imprecise Previsions For Risk Measurement},
  journal      = {Int. J. Uncertain. Fuzziness Knowl. Based Syst.},
  volume       = {11},
  number       = {4},
  pages        = {393--412},
  year         = {2003},
}

@book{royall2017statistical,
  title={Statistical evidence: a likelihood paradigm},
  author={Royall, Richard},
  year={1997},
  publisher={Routledge}
}

@article{hofmanEfficientCredalPrediction2026,
  author       = {Paul Hofman and
                  Timo L{\"{o}}hr and
                  Maximilian Muschalik and
                  Yusuf Sale and
                  Eyke H{\"{u}}llermeier},
  title        = {Efficient Credal Prediction through Decalibration},
  journal      = {CoRR},
  volume       = {abs/2603.08495},
  year         = {2026},
}

@article{sengeReliableClassification2014,
  author       = {Robin Senge and
                  Stefan B{\"{o}}sner and
                  Krzysztof Dembczynski and
                  J{\"{o}}rg Haasenritter and
                  Oliver Hirsch and
                  Norbert Donner{-}Banzhoff and
                  Eyke H{\"{u}}llermeier},
  title        = {Reliable classification: Learning classifiers that distinguish aleatoric
                  and epistemic uncertainty},
  journal      = {Inf. Sci.},
  volume       = {255},
  pages        = {16--29},
  year         = {2014},
}

@inproceedings{wangCredalDistillation2026,
  author       = {Kaizheng Wang and
                  Fabio Cuzzolin and
                  David Moens and
                  Hans Hallez},
  title        = {Credal Ensemble Distillation for Uncertainty Quantification},
  booktitle    = {Fortieth {AAAI} Conference on Artificial Intelligence, Thirty-Eighth
                  Conference on Innovative Applications of Artificial Intelligence,
                  Sixteenth Symposium on Educational Advances in Artificial Intelligence,
                  {AAAI}},
  pages        = {26319--26327},
  publisher    = {{AAAI} Press},
  year         = {2026},
}

@article{wassermanBelief1990,
 author = {Larry A. Wasserman},
 journal = {The Canadian Journal of Statistics / La Revue Canadienne de Statistique},
 number = {3},
 pages = {183--196},
 publisher = {[Statistical Society of Canada, Wiley]},
 title = {Belief Functions and Statistical Inference},
 urldate = {2026-07-14},
 volume = {18},
 year = {1990}
}

@article{denoeuxBelief2014,
title = {Likelihood-based belief function: Justification and some extensions to low-quality data},
journal = {International Journal of Approximate Reasoning},
volume = {55},
number = {7},
pages = {1535-1547},
year = {2014},
note = {Special issue: Harnessing the information contained in low-quality data sources},
author = {Thierry Denœux},
}

@inproceedings{antonucciLikelihood2012,
  title={Likelihood-based robust classification with Bayesian networks},
  author={Antonucci, Alessandro and Cattaneo, Marco EGV and Corani, Giorgio},
  booktitle={International Conference on Information Processing and Management of Uncertainty in Knowledge-Based Systems},
  pages={491--500},
  year={2012},
  organization={Springer}
}

@article{lohrCredal2025,
  title={Credal prediction based on relative likelihood},
  author={L{\"o}hr, Timo and Hofman, Paul and Mohr, Felix and H{\"u}llermeier, Eyke},
  journal={Advances in Neural Information Processing Systems},
  volume={38},
  pages={145184--145213},
  year={2025}
}

@article{duboisPossibility1997,
  title={A semantics for possibility theory based on likelihoods},
  author={Dubois, Didier and Moral, Seraf{\i}n and Prade, Henri},
  journal={Journal of Mathematical analysis and applications},
  volume={205},
  number={2},
  pages={359--380},
  year={1997},
  publisher={Elsevier}
}

@article{wangCreNet2025,
  title={Creinns: Credal-set interval neural networks for uncertainty estimation in classification tasks},
  author={Wang, Kaizheng and Shariatmadar, Keivan and Manchingal, Shireen Kudukkil and Cuzzolin, Fabio and Moens, David and Hallez, Hans},
  journal={Neural networks},
  volume={185},
  pages={107198},
  year={2025},
  publisher={Elsevier}
}

@article{abellanUncertaintyMeasure2006,
  author       = {Joaqu{\'{\i}}n Abell{\'{a}}n and
                  George J. Klir and
                  Seraf{\'{\i}}n Moral},
  title        = {Disaggregated total uncertainty measure for credal sets},
  journal      = {Int. J. Gen. Syst.},
  volume       = {35},
  number       = {1},
  pages        = {29--44},
  year         = {2006},
}

@inproceedings{liOODWrongQuestions2025,
  author       = {Yucen Lily Li and
                  Daohan Lu and
                  Polina Kirichenko and
                  Shikai Qiu and
                  Tim G. J. Rudner and
                  C. Bayan Bruss and
                  Andrew Gordon Wilson},
  title        = {Position: Supervised Classifiers Answer the Wrong Questions for {OOD}
                  Detection},
  booktitle    = {Forty-second International Conference on Machine Learning, {ICML}},
  publisher    = {OpenReview.net},
  year         = {2025},
}

@book{waldStasticalDecisionFunctions1950,
  author       = {Abraham Wald},
  title        = {Statistical Decision Functions},
  publisher    = {John Wiley},
  year         = {1950},
}

@inproceedings{wangCredalWrapper2025,
 author = {Wang, Kaizheng and Cuzzolin, Fabio and Shariatmadar, Keivan and Moens, David and Hallez, Hans},
 booktitle = {International Conference on Learning Representations},
 editor = {Y. Yue and A. Garg and N. Peng and F. Sha and R. Yu},
 pages = {70290--70318},
 title = {Credal Wrapper of Model Averaging for Uncertainty Estimation in Classification},
 year = {2025}
}

@article{gilboa1989maxmin,
  title={Maxmin expected utility with non-unique prior},
  author={Gilboa, Itzhak and Schmeidler, David},
  journal={Journal of mathematical economics},
  volume={18},
  number={2},
  pages={141--153},
  year={1989},
  publisher={Elsevier}
}

@book{berger2013statistical,
  title={Statistical decision theory and Bayesian analysis},
  author={Berger, James O},
  year={2013},
  publisher={Springer Science \& Business Media}
}

@article{troffaes2007decision,
  title={Decision making under uncertainty using imprecise probabilities},
  author={Troffaes, Matthias CM},
  journal={International journal of approximate reasoning},
  volume={45},
  number={1},
  pages={17--29},
  year={2007},
  publisher={Elsevier}
}

@article{nguyen2025credal,
  title={Credal ensembling in multi-class classification},
  author={Nguyen, Vu-Linh and Zhang, Haifei and Destercke, S{\'e}bastien},
  journal={Machine Learning},
  volume={114},
  number={1},
  pages={19},
  year={2025},
  publisher={Springer}
}

@article{caprio2023credal,
  title={Credal Bayesian deep learning},
  author={Caprio, Michele and Dutta, Souradeep and Jang, Kuk Jin and Lin, Vivian and Ivanov, Radoslav and Sokolsky, Oleg and Lee, Insup},
  journal={arXiv preprint arXiv:2302.09656},
  year={2023}
}

@article{wangDro2026,
  title={Learning Credal Ensembles via Distributionally Robust Optimization},
  author={Wang, Kaizheng and Faza, Ghifari Adam and Cuzzolin, Fabio and Chau, Siu Lun and Moens, David and Hallez, Hans},
  journal={arXiv preprint arXiv:2602.08470},
  year={2026}
}

@article{rigterRLEpistemic2023,
  title={One risk to rule them all: A risk-sensitive perspective on model-based offline reinforcement learning},
  author={Rigter, Marc and Lacerda, Bruno and Hawes, Nick},
  journal={Advances in neural information processing systems},
  volume={36},
  pages={77520--77545},
  year={2023}
}

@article{iyengar2005robust,
  title={Robust dynamic programming},
  author={Iyengar, Garud N},
  journal={Mathematics of Operations Research},
  volume={30},
  number={2},
  pages={257--280},
  year={2005},
  publisher={INFORMS}
}

@article{wiesemann2013robust,
  title={Robust Markov decision processes},
  author={Wiesemann, Wolfram and Kuhn, Daniel and Rustem, Ber{\c{c}}},
  journal={Mathematics of Operations Research},
  volume={38},
  number={1},
  pages={153--183},
  year={2013},
  publisher={INFORMS}
}

@InProceedings{kiyani2025decision,
  title = 	 {Decision Theoretic Foundations for Conformal Prediction: Optimal Uncertainty Quantification for Risk-Averse Agents},
  author =       {Kiyani, Shayan and Pappas, George J. and Roth, Aaron and Hassani, Hamed},
  booktitle = 	 {Proceedings of the 42nd International Conference on Machine Learning},
  pages = 	 {30943--30965},
  year = 	 {2025},
  volume = 	 {267},
  series = 	 {Proceedings of Machine Learning Research},
  month = 	 {13--19 Jul},
  publisher =    {PMLR},
}

@inproceedings{heResNet2016,
  author       = {Kaiming He and
                  Xiangyu Zhang and
                  Shaoqing Ren and
                  Jian Sun},
  title        = {Deep Residual Learning for Image Recognition},
  booktitle    = {2016 {IEEE} Conference on Computer Vision and Pattern Recognition},
  pages        = {770--778},
  publisher    = {{IEEE} Computer Society},
  year         = {2016},
}

@article{zhangOpenOOD2024,
  author       = {Jingyang Zhang and
                  Jingkang Yang and
                  Pengyun Wang and
                  Haoqi Wang and
                  Yueqian Lin and
                  Haoran Zhang and
                  Yiyou Sun and
                  Xuefeng Du and
                  Yixuan Li and
                  Ziwei Liu and
                  Yiran Chen and
                  Hai Li},
  title        = {OpenOOD v1.5: Enhanced Benchmark for Out-of-Distribution Detection},
  journal      = {J. Data-centric Mach. Learn. Res.},
  volume       = {2},
  pages        = {(3):1--32},
  year         = {2024},
}

@article{yangMedMNIST2021,
  author       = {Jiancheng Yang and
                  Rui Shi and
                  Donglai Wei and
                  Zequan Liu and
                  Lin Zhao and
                  Bilian Ke and
                  Hanspeter Pfister and
                  Bingbing Ni},
  title        = {MedMNIST v2: {A} Large-Scale Lightweight Benchmark for 2D and 3D Biomedical
                  Image Classification},
  journal      = {CoRR},
  volume       = {abs/2110.14795},
  year         = {2021},
}

@inproceedings{wilsonDeepKernelLearning2016,
  author       = {Andrew Gordon Wilson and
                  Zhiting Hu and
                  Ruslan Salakhutdinov and
                  Eric P. Xing},
  title        = {Deep Kernel Learning},
  booktitle    = {Proceedings of the 19th International Conference on Artificial Intelligence
                  and Statistics},
  series       = {{JMLR} Workshop and Conference Proceedings},
  volume       = {51},
  pages        = {370--378},
  publisher    = {JMLR.org},
  year         = {2016},
}

@article{parzenEstimationProbabilityDensity1962,
  title={On estimation of a probability density function and mode},
  author={Parzen, Emanuel},
  journal={The annals of mathematical statistics},
  volume={33},
  number={3},
  pages={1065--1076},
  year={1962},
  publisher={JSTOR}
}

@inproceedings{leeMahalanobisOOD2018,
  author       = {Kimin Lee and
                  Kibok Lee and
                  Honglak Lee and
                  Jinwoo Shin},
  title        = {A Simple Unified Framework for Detecting Out-of-Distribution Samples
                  and Adversarial Attacks},
  booktitle    = {Advances in Neural Information Processing Systems 31: Annual Conference
                  on Neural Information Processing Systems},
  pages        = {7167--7177},
  year         = {2018},
}

@article{tibshiraniLocalLikelihood1987,
  title={Local likelihood estimation},
  author={Tibshirani, Robert and Hastie, Trevor},
  journal={Journal of the American Statistical Association},
  volume={82},
  number={398},
  pages={559--567},
  year={1987},
  publisher={Taylor \& Francis}
}

@book{loaderLocalLikelihood2006,
  title={Local regression and likelihood},
  author={Loader, Catherine},
  year={2006},
  publisher={Springer Science \& Business Media}
}

@inproceedings{hendrycksBenchmarkingCorruptions2019,
  author       = {Dan Hendrycks and
                  Thomas G. Dietterich},
  title        = {Benchmarking Neural Network Robustness to Common Corruptions and Perturbations},
  booktitle    = {7th International Conference on Learning Representations, {ICLR}},
  publisher    = {OpenReview.net},
  year         = {2019},
}

@inproceedings{petersonCIFAR10H2019,
  author       = {Joshua C. Peterson and
                  Ruairidh M. Battleday and
                  Thomas L. Griffiths and
                  Olga Russakovsky},
  title        = {Human Uncertainty Makes Classification More Robust},
  booktitle    = {{IEEE/CVF} International Conference on Computer Vision, {ICCV}},
  pages        = {9616--9625},
  publisher    = {{IEEE}},
  year         = {2019},
}

@article{radovanovicHubsInSpace2010,
  author       = {Milos Radovanovic and
                  Alexandros Nanopoulos and
                  Mirjana Ivanovic},
  title        = {Hubs in Space: Popular Nearest Neighbors in High-Dimensional Data},
  journal      = {J. Mach. Learn. Res.},
  volume       = {11},
  pages        = {2487--2531},
  year         = {2010},
}

@inproceedings{sculleyMiniBatchKMeans2010,
  author       = {D. Sculley},
  editor       = {Michael Rappa and
                  Paul Jones and
                  Juliana Freire and
                  Soumen Chakrabarti},
  title        = {Web-scale k-means clustering},
  booktitle    = {Proceedings of the 19th International Conference on World Wide Web,
                  {WWW}},
  pages        = {1177--1178},
  publisher    = {{ACM}},
  year         = {2010},
}

@inproceedings{kingmaAdam2015,
  author       = {Diederik P. Kingma and
                  Jimmy Ba},
  title        = {Adam: {A} Method for Stochastic Optimization},
  booktitle    = {3rd International Conference on Learning Representations, {ICLR}},
  year         = {2015},
}

@article{finlaysonDatasetShift2021,
author = {Samuel G. Finlayson  and Adarsh Subbaswamy  and Karandeep Singh  and John Bowers  and Annabel Kupke  and Jonathan Zittrain  and Isaac S. Kohane  and Suchi Saria },
title = {The Clinician and Dataset Shift in Artificial Intelligence},
journal = {New England Journal of Medicine},
volume = {385},
number = {3},
pages = {283-286},
year = {2021},
}

@inproceedings{kohWilds2021,
  author       = {Pang Wei Koh and
                  Shiori Sagawa and
                  Henrik Marklund and
                  Sang Michael Xie and
                  Marvin Zhang and
                  Akshay Balsubramani and
                  Weihua Hu and
                  Michihiro Yasunaga and
                  Richard Lanas Phillips and
                  Irena Gao and
                  Tony Lee and
                  Etienne David and
                  Ian Stavness and
                  Wei Guo and
                  Berton Earnshaw and
                  Imran S. Haque and
                  Sara M. Beery and
                  Jure Leskovec and
                  Anshul Kundaje and
                  Emma Pierson and
                  Sergey Levine and
                  Chelsea Finn and
                  Percy Liang},
  title        = {{WILDS:} {A} Benchmark of in-the-Wild Distribution Shifts},
  booktitle    = {Proceedings of the 38th International Conference on Machine Learning,
                  {ICML}},
  series       = {Proceedings of Machine Learning Research},
  volume       = {139},
  pages        = {5637--5664},
  publisher    = {{PMLR}},
  year         = {2021},
}

@inproceedings{ovadiaDatasetShift2019,
  author       = {Yaniv Ovadia and
                  Emily Fertig and
                  Jie Ren and
                  Zachary Nado and
                  David Sculley and
                  Sebastian Nowozin and
                  Joshua V. Dillon and
                  Balaji Lakshminarayanan and
                  Jasper Snoek},
  title        = {Can you trust your model's uncertainty? Evaluating predictive uncertainty
                  under dataset shift},
  booktitle    = {Advances in Neural Information Processing Systems 32: Annual Conference
                  on Neural Information Processing Systems},
  pages        = {13969--13980},
  year         = {2019},
}

@article{foellmer2011entropic,
  title={Entropic risk measures: Coherence vs. convexity, model ambiguity and robust large deviations},
  author={F{\"o}llmer, Hans and Knispel, Thomas},
  journal={Stochastics and Dynamics},
  volume={11},
  number={02n03},
  pages={333--351},
  year={2011},
  publisher={World Scientific}
}

@article{cuzzolinIntersectionProbability2022,
  author       = {Fabio Cuzzolin},
  title        = {The intersection probability: betting with probability intervals},
  journal      = {CoRR},
  volume       = {abs/2201.01729},
  year         = {2022},
  url          = {https://doi.org/10.48550/arXiv.2201.01729},
  doi          = {10.48550/ARXIV.2201.01729},
  eprinttype   = {arXiv},
  eprint       = {2201.01729},
  bibsource    = {dblp computer science bibliography, https://dblp.org}
}

    \clearpage
    \appendix

    \crefalias{section}{appendix}
    \crefalias{subsection}{appendix}
    \crefalias{subsubsection}{appendix}

    \begin{center}
        {\LARGE\scshape Credal Machine Learning for\\Risk-Averse Decision Making: \\
        Supplementary Materials\par}
    \end{center}

    \addtocontents{toc}{\protect\setcounter{tocdepth}{2}}
    \setcounter{tocdepth}{2}
    \tableofcontents

    \clearpage

    \section{Variations on the CVaR}\label{app:variants}
    The loss $L = \ell(h(X), Y)$ inherits its randomness from two sources: the draw of the instance $X \sim P_X$ and, fixing an instance $\vx$, the draw of the outcome $Y \sim p^*(\cdot \mid \vx)$. The tail in \cref{eq:cvar} is taken with respect to $L$ under the joint measure $P$ --- marginalizing over instances and outcomes --- and we therefore refer to \cref{eq:cvar} as the \emph{marginal} \ac{cvar}. We make this distinction, because the two sources of randomness can also be aggregated separately, and different orders of aggregation yield different risk functionals. A first alternative applies the \ac{cvar} to the conditional loss $L_{\vx} = \ell(h(\vx), Y)$ with $Y \sim p^*(\cdot \mid \vx)$, and averages the resulting instance-wise tail risks
    \begin{equation}\label{eq:cvar-instance-wise}
        \bE_{P_X}\left[\cvar_\tau(L_X)\right].
    \end{equation}
    This quantity captures the \emph{aleatoric} tail risk locally, but, crucially, its estimation relies on the presence of ground-truth conditional distributions, which are usually not available in machine learning. In econometrics, \cref{eq:cvar-instance-wise} arises as the natural estimand of expected shortfall regression \citep{dimitriadisExpectedShortfall2019,pattonDynamicModelsES2019}.

    For our purposes, however, the outer averaging over instances in \cref{eq:cvar-instance-wise} is precisely the problem. A large \ac{cvar} at one instance can be compensated by a small \ac{cvar} at another instance. In practice, e.g., in a medical context, this means that large errors on a few patients may be offset by small errors on the remaining ones. The marginal \ac{cvar} permits no such compensation: any sufficiently large mass of high losses enters the tail of the loss distribution, it does not matter whether it is spread across many instances or concentrated on a single, systematically mishandled one.

    The composition in \cref{eq:cvar-instance-wise} can also be transposed. Applying the \ac{cvar} to the expected risks $\bE_{Y \sim p^*(\cdot \mid \vx)}\left[L_{\vx}\right]$ across instances or subgroups yields worst-case subpopulation risks, which is used in fairness \citep{williamsonFairnessRiskMeasures2019,laguelSuperquantileBasedLearning2022} and group-robust learning \citep{sagawaDROGroupShifts2019}.

    In this work, we deliberately consider the marginal \ac{cvar}. In safety-critical applications, what matters is that large errors are rare in the population as a whole, irrespective of their origin. In the remainder of the paper, we therefore refer to \cref{eq:cvar} and its estimate \cref{eq:cvar-empirical} when discussing \ac{cvar}.

    \clearpage

    \section{Extended Related Work}
    \label{app:related-work}
    Credal sets originate in the imprecise probability literature as a representation of partial ignorance \citep{walleyCredal1991,leviCredal1978,augustin2014introduction}, and have more recently been adopted in \ac{ml} as a formalism for epistemic uncertainty \citep{huellermeierUncertainty2021,coraniBnn2012}. In \ac{ml}, credal predictors are typically built at training time, by training networks to emit a lower and an upper probability per class \citep{wangCreNet2025}, by instantiating BNNs with a set of priors \citep{caprio2023credal}, or by extracting class-wise intervals from a collection of ensemble predictions \citep{wangCredalWrapper2025,nguyen2025credal,wangDro2026}. A second line of work constructs credal sets based on relative likelihood, namely as an $\alpha$-cut of the normalized likelihood function. Before the notion was used in \ac{ml}, it appeared in statistics as likelihood-based confidence regions \citep{wilksLikelihood1938}, in the reading of the likelihood as a belief function \citep{shaferLikelihood1976,wassermanBelief1990,denoeuxBelief2014}, and in likelihood-based decision functions and their $\alpha$-cuts \citep{antonucciLikelihood2012}, with a possibilistic view in \citet{duboisPossibility1997}. Recent methods instantiate it in neural networks \citep{lohrCredal2025,sluijtermanLikelihoodRatioNNs2025,niPossibilisticPredictiveUncertainty2026}. These existing methods derive the credal set from disagreement among models trained on the same data, so that its width is governed by the training procedure rather than by the characteristics of the test instance. Whether the resulting set is useful is a question rarely posed in decision-theoretic terms, evaluation instead relying on proxies such as \ac{ood} detection, whose suitability has been questioned \citep{liOODWrongQuestions2025}. What distinguishes our construction is that the set is fit post-hoc and per test instance.

    Acting on a credal set requires a decision rule, and several have been proposed. The principal candidates are $\Gamma$-maximin \citep{gilboa1989maxmin}, maximality \citep{walleyCredal1991}, and E-admissibility \citep{leviCredal1978}, surveyed by \citet{troffaes2007decision}. These rules differ in how they select a distribution from the set, but agree in what they do afterwards, in that outcomes are aggregated by expectation. Risk aversion is therefore exercised at one level only. The rule hedges against ignorance about the ground-truth distribution, yet it remains indifferent to the tail of the outcome distribution once a model has been fixed, and thus accepts a known but severe risk whenever that risk is favorable on average (cf. \cref{sec:risk-averse-prediction}). A related perspective replaces the set of distributions by a set of labels, with \citet{kiyani2025decision} studying worst-case loss minimization over conformal prediction sets.

    The second, aleatoric perspective is the subject of a largely separate literature. The \ac{cvar} is the most common risk measure in \ac{ml} \citep{rockafellarOptimizationCVaR2000,artznerCoherentRiskMeasures1999}, minimized directly \citep{fanLearningTopkLoss2017,curiAdaptiveSampling2020,zhaiBoostedCVaRClassification2021} or as an instance of a \ac{dro} problem whose ambiguity set is a reweighting of the empirical distribution \citep{duchiLearningModelsDRO2018,sagawaDROGroupShifts2019,kuhnDRO2025,chenCredalAmbiguitySets2026}. In this sense the \ac{cvar} is itself a worst case over a set of distributions. That set, however, reweights the loss the data already exhibit rather than ranging over models, and is therefore orthogonal to the credal set of the previous paragraphs. Keeping the two apart is what allows the two forms of risk aversion to be governed by distinct parameters, $\tau$ for the outcome tail and $\alpha$ for the credal set. The same distinction recurs in sequential settings, where risk-averse methods target the tail of the return under a known model \citep{chowCVaROptimization2015,greenbergRiskAverseRL2022}, while robust Markov decision processes optimize against a set of plausible transition models at each state \citep{iyengar2005robust,wiesemann2013robust}. Safety under epistemic uncertainty thus requires both. Our decision rule based on credal sets combines them, and the case study in \cref{sec:rl-driving} illustrates that our two-level distinction governs sequential decisions, exhibiting the safety behavior of specialized RL methods. Closest among these is \citet{rigterRLEpistemic2023}, which likewise addresses both sources, though through a second-order posterior over models rather than the prior-free credal set used here.

    \clearpage

    \section{Proofs and Derivations}

    \subsection{Minimizer of Expected Risk versus CVaR}\label{app:proof-minimizers}
    Given $\vx$, for brevity, we write $p^*$ for $p^*(\cdot \mid \vx)$ and $p^*(y)$ for $p^*(y \mid \vx)$, denote the expected loss as $\bell(p,q) \coloneqq \bE_{Y \sim q}\left[\ell(p, Y)\right]$, and the objective of \cref{eq:cvar-distribution} as $J(p) \coloneqq \bE_{Y \sim p^*}\left[(\ell(p, Y) - v^*)_+ \right]$. Due to strict propriety $\bell(\cdot, q)$ is minimized by $q$ alone, for every $q \in \Delta_K$ \citep{gneitingProperScoringRules2007}. Additionally, by tail non-degeneracy, the tail-conditional distribution $q_{\cT} \coloneqq p^*(y) \llbracket y \in \cT_{v^*} \rrbracket / p^*(\cT_{v^*})$ is well defined, and $q_{\cT} \neq p^*$, since $q_{\cT}(\cT_{v^*}) = 1 > p^*(\cT_{v^*})$.

    \paragraph{Regularity conditions}

    \begin{enumerate}
        \item[(R1)] \textit{Full support:} $p^*(y \mid \vx) > 0$ for all $y \in \cY$.
        \item[(R2)] \textit{Finiteness and continuity:} $\ell$ is real-valued on $\operatorname{relint}(\Delta_K) \times \cY$ and each map $p \mapsto \ell(p,y)$ is continuous at $p^*$.
        \item[(R3)] \textit{Threshold regularity:} $v^* \in \bR$ and $\ell(p^*, y) \neq v^*$ for all $y \in \cY$.
    \end{enumerate}

    \begin{proof}
        Let $u \coloneqq q_{\cT} - p^* \neq 0$ and $p_t \coloneqq p^* + tu = (1-t)p^* +tq_{\cT}$, which lies in $\operatorname{relint}(\Delta_K)$ for $t \in [0,1)$ by (R1). Note that $u$ is not a probability vector but a vector in $\bR^K$. By (R2) and (R3), $\ell(p_t,y) \rightarrow \ell(p^*,y)$ as $t \rightarrow 0$ with $\ell(p^*,y) > v^*$ for $y \in \cT_{v^*}$ and $\ell(p^*,y) < v^*$ otherwise, hence there is a $t_0 \in (0,1)$ such that, for all $t \in [0, t_0)$, the hinge is active at $p_t$ on exactly the tail classes:
        \begin{equation}\label{eq:objective}
            J(p_t) = \sum_{y \in \cT_{v^*}} p^*(y)(\ell(p_t, y) - v^*) = p^*(\cT_{v^*})\left[\bell(p_t,q_{\cT}) - v^* \right].
        \end{equation}
        We still have to show that $\varphi(t) \coloneqq \bell(p_t,q_{\cT})$ is strictly decreasing on $[0,1)$. Let $\psi(t) \coloneqq \bell(p_t,p_t)$ denote the generalized entropy of $p_t$ \citep{gneitingProperScoringRules2007}. For fixed $t \in [0,1)$, the map $s \mapsto \bell(p_t,p_s)$ is affine in $s$ and, by strict propriety, dominates $\psi$ with equality exactly at $s = t$. In other words, it is a line supporting $\psi$ from above, with slope $m(t) \in \bR$, and evaluating it at $s = 1$ gives
        \begin{equation*}
            \varphi(t) = \psi(t) + (1-t)m(t).
        \end{equation*}
        Given $0 \leq t_1 < t_2 < 1$, evaluating the supporting line at $t_2$ in the point $t_1$, and vice versa, yields
        \begin{equation*}
            (t_2 - t_1) m(t_2) \leq \psi(t_2) - \psi(t_1) \leq (t_2 - t_1) m(t_1),
        \end{equation*}
        hence $m(t_2) \leq m(t_1)$. Substituting $\psi(t_1) = \bell(p_{t_1}, p_{t_2}) - (t_2 - t_1)m(t_1)$ then gives
        \begin{equation*}
            \varphi(t_2) - \varphi(t_1) = \left[\psi(t_2) - \bell(p_{t_1}, p_{t_2})\right] + (1 - t_2)(m(t_2) - m(t_1)) < 0,
        \end{equation*}
        since the bracket is negative by strict propriety $(p_{t_1} \neq p_{t_2})$ and the second term is non-positive. Combining this with \cref{eq:objective} and $p^*(\cT_{v^*}) > 0$,
        \begin{equation*}
            J(p_t) = p^*(\cT_{v^*}) \left[\varphi(t) - v^*\right] < p^*(\cT_{v^*}) \left[\varphi(0) - v^* \right] = J(p^*) \quad \text{for all } t \in (0, t_0),
        \end{equation*}
        thus $p^* \notin \argmin_{p\in\Delta_K} J(p)$. Since $p_t \rightarrow p^*$, $p^*$ is not even a local minimizer.
    \end{proof}

    \subsection{Derivations of Examples}\label{app:examples}
    We provide examples with derivations of how the minimizer changes when moving from expected risk to \ac{cvar}.

    \subsubsection{Log-loss}
    In the setting of \cref{sec:minimizing-cvar}, we have $p^* = (0.8, 0.2)$ and $v = \frac{1}{2}$, such that
    \begin{equation*}
        J(p(y_1)) = 0.8 \left(-\log p(y_1) - \frac{1}{2}\right)_+ + 0.2 \left(-\log(1-p(y_1)) - \frac{1}{2}\right)_+.
    \end{equation*}
    The hinge of the first term is active iff $p(y_1) < e^{-1/2}$, and the hinge of the second term is active iff $p(y_1) > 1 - e^{-1/2}$, so $J$ is convex and piecewise smooth with breakpoints at these two points. At $p(y_1) = e^{-1/2}$, the one-sided derivatives are
    \begin{equation*}
        J'_-(e^{-1/2}) = -\frac{0.8}{p(y_1)} + \frac{0.2}{1-p(y_1)} \approx -0.81 < 0,
        \qquad
        J'_+(e^{-1/2}) = \frac{0.2}{1-p(y_1)} \approx 0.51 > 0,
    \end{equation*}
    hence, by convexity, the minimum is attained exactly there: $\tp = (e^{-1/2}, 1 - e^{-1/2})$. The probability of the most likely class is thus reduced until its loss reaches the point at which the hinge would become active. Note also that the minimizer remains unchanged for every $p^*(y_1) > e^{-1/2}$.

    \subsubsection{Brier score}
    Let $\ell_{\rm{Brier}}(p, y) = \sum_{k=1}^K \left(p(k) - \llbracket y = k \rrbracket\right)^2$, so that in the binary case $\ell_{\rm{Brier}}(p, y_1) = 2(1 - p(y_1))^2$ and $\ell_{\rm{Brier}}(p, y_2) = 2p(y_1)^2$. With $p^* = (0.8, 0.2)$ and $v = \frac{1}{4}$,
    \begin{equation*}
        J(p(y_1)) = 0.8 \left(2(1-p(y_1))^2 - \frac{1}{4}\right)_+ + 0.2 \left(2p(y_1)^2 - \frac{1}{4}\right)_+.
    \end{equation*}
    The hinge of the first term is active iff $p(y_1) < 1 - \frac{1}{2\sqrt{2}}$, and the hinge of the second term is active iff $p(y_1) > \frac{1}{2\sqrt{2}}$. On the interval where both are active, $J'(p(y_1)) = -3.2\,(1 - p(y_1)) + 0.8\,p(y_1) < 0$, since the unconstrained minimizer $p(y_1) = 0.8$ lies to its right, and beyond $1 - \frac{1}{2\sqrt{2}}$ only the second term remains, which is increasing. Hence $\tp = \left(1 - \frac{1}{2\sqrt{2}}, \frac{1}{2\sqrt{2}}\right) \approx (0.65, 0.35)$. As for the log-loss, the likely class is pushed down until its loss reaches the threshold, and the minimizer is $1 - \sqrt{v/2}$ for every $p^*(y_1) > 1 - \sqrt{v/2}$. Note that for the one-coordinate binary Brier score $(p(y_1) - \llbracket y = y_1 \rrbracket)^2$, that is, half of the above, the same threshold $v = \frac{1}{4}$ coincides with the minimax value $\min_p \max_y \ell(p, y)$, so that both hinges are inactive at $p(y_1) = \frac{1}{2}$ and the uniform distribution is the (unique) minimizer with $J = 0$.

    \clearpage

    \section{Algorithmic Details}\label{app:practical-implementation}
    Our novel credal predictor consists of two distinct procedures. Given an encoder $f : \mathcal{X} \to \mathcal{Z}, \vx \mapsto \vz$ or a pre-trained network $h(\vx) = \softmax(g(f(\vx)))$ that can be reduced to an encoder, e.g., by removing the linear head $g$, using the method requires a fitting and prediction step.

    \paragraph{Fitting.}
    The credal predictor operates post hoc and fitting reduces to a single forward pass of the encoder over the training data to generate training embeddings $\vz_i = f(\vx_i) \in \bR^d$ with $d$ some embedding dimension. The embeddings are generated with test-time transforms and are L2-normalized and whitened. Concretely, let $\Sigma$ denote the class-wise centered (pooled within-class) covariance of the normalized embeddings, with ridge regularization of $10^{-2}$ times its mean diagonal entry. The embeddings are mapped by $W = \Sigma^{-\gamma/2}$, where $\gamma = 1$ corresponds to full Mahalanobis whitening and $\gamma =0$ to no whitening ($W = I$). Using full whitening sharpens the separation of far-away instances, but can induce ``hubness'' \citep{radovanovicHubsInSpace2010}, causing a few embeddings to become the nearest neighbors of most instances. Therefore, $\gamma$ is selected at fit time as the largest value on the grid $\{0, 0.25, 0.5, 0.75, 1.0\}$ for which the leave-one-out kernel vote on the training data remains healthy, that is, its accuracy is within $0.02$ of the accuracy at $\gamma = 0$, and the share of the most frequently predicted class exceeds the true majority share by at most $0.1$. The kernel is Gaussian, specifically $\kappa(\vz, \vz') = \exp(-\lVert \vz - \vz' \rVert^2 / 2\sigma^2)$, with the bandwidth $\sigma$ set to $0.55$ times the median distance of a reference point to its nearest other reference point in whitened space. The fitted predictor thus stores the whitened embeddings with their labels (in float16 format), the map $W$, and $\sigma$, concluding the fitting step.

    \paragraph{Prediction.}
    Given a test instance $\vx$, its embedding is normalized, whitened, and compared to the $200$ nearest reference points. The weights are accumulated into the class-wise evidence $e_k(\vx)$, the total evidence $e(\vx)$, and the local maximum likelihood estimate $\hat{p}(\cdot \mid \vx)$ as described in \cref{sec:credal}. Profiling the multinomial likelihood over the remaining classes collapses the local relative likelihood of class $k$ to the binomial
    \begin{equation*}
        \Lambda_{\vx,k}(p) = \left(\frac{p}{\hat{p}(k \mid \vx)}\right)^{e_k(\vx)} \left(\frac{1 - p}{1 - \hat{p}(k \mid \vx)}\right)^{e(\vx) - e_k(\vx)},
    \end{equation*}
    which increases on $[0, \hat{p}(k \mid \vx)]$ and decreases on $[\hat{p}(k \mid \vx), 1]$, such that $\lowp(k \mid \vx)$ and $\uppp(k \mid \vx)$ are the two solutions of $\Lambda_{\vx,k}(p) = \alpha$. The solutions are obtained using a simple bisection of $50$ steps per bound. Prediction, therefore, costs one forward pass, one $N \times d$ distance computation, and $K$ fixed-step bisections per instance.

    \paragraph{Efficiency and embedding compression.}
    The most costly operation at test time is the distance computation against all $N$ stored reference embeddings. In the settings considered in this paper, the cost is rather modest --- for example, for CIFAR-10, $N=45000$, and $d = 512$, but it scales linearly with training set size. However, this dependence can be easily removed by compressing the reference embeddings into $M \ll N$ \emph{landmarks}. The whitened embeddings are partitioned by mini-batch $k$-means \citep{sculleyMiniBatchKMeans2010} with batch size of $4096$, three initializations, and at most $60$ passes. Each landmark $m$ stores its centroid $c_m$ together with the class histogram of its cluster, $n_{m,k} = \sum_i \llbracket a_i = m \rrbracket \llbracket y_i = k \rrbracket$ with $a_i$ the cluster assignment. The evidence is then computed with respect to the landmarks,
    \begin{equation*}
        e_k(\vx) = \sum_{m=1}^{M} \kappa(\vz, c_m)\ n_{m,k},
    \end{equation*}
    at a cost of $M \cdot d$ per instance. We study the influence of the number of landmarks $M$ on the train time, test time, and performance of the credal predictor in terms of \ac{ood} detection. We fit the algorithm with $M \in \{45000, 20000, 10000, 5000, 2000, 1000\}$ over three seeds and show the results in \cref{fig:efficiency}. As the number of landmark points $M$ decreases, the fitting step proceeds faster. Note, however, that in terms of total training time (training + fitting) our method was already on par with the most efficient credal prediction method (see \cref{fig:time}). Furthermore, the inference time, while noisy, does not decrease significantly, and, importantly, the AUROC remains largely the same. This highlights that for the scale of typical image benchmark datasets, our method does not require any compression of the reference embeddings.

    \begin{figure}[h]
        \centering
        \includegraphics[width=\linewidth]{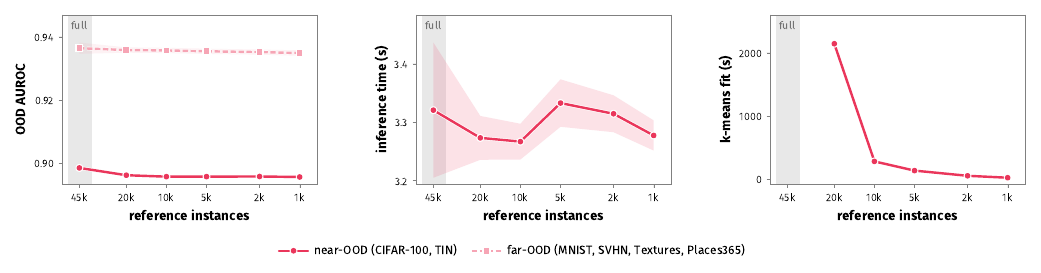}
        \caption{AUROC, fitting time, and inference time (in seconds) of CreWare on CIFAR-10 as a function of the number of landmarks $M \in \{45000, 20000, 10000, 5000, 2000, 1000\}$. Mean and standard deviation over three seeds are shown.}
        \label{fig:efficiency}
    \end{figure}

    \paragraph{Finding the CVaR-minimax distribution.}
    Given a credal set $\cQ$, what remains is to find the \ac{cvar}-minimax prediction \cref{eq:cvar-distribution-set}. For a fixed threshold $v$, the inner supremum in \cref{eq:cvar-distribution-set} is linear in $q$ and can hence be solved exactly: by water-filling for interval-based credal sets --- starting every class at its lower probability and assigning the remaining mass to the classes with the largest hinged loss, each up to its upper probability --- and by enumeration of the vertices for convex hulls. The outer minimization over $p \in \Delta_K$ is convex. It is solved by $300$ iterations of Adam \citep{kingmaAdam2015} with learning rate $0.1$ on the parameterization $p = \softmax(\vz)$, where, by Danskin's theorem, differentiating through the exact inner solution yields the gradient of the worst-case objective. The solver is warm-started at the maximum-entropy element of $\cQ_{\vx}$, which is the exact solution for $v = 0$ under log loss. For the calibration in \cref{alg:threshold} we take $\cV$ to be a grid of $100$ equally spaced values on $[0, v_{\max}]$, with $v_{\max}$ the largest hold-out loss $L_i$ of the maximum-entropy prediction, beyond which the hinge vanishes and $C(v) = v$ can only increase. Since $\tp_v$ does not depend on $\tau$, a single sweep over $\cV$ serves all risk levels.

    \clearpage

    \section{Implementation Details}\label{app:implementation}

    \subsection{Compute Resources}
    We provide an overview of the compute resources used to perform the experiments in \cref{tab:hardware}.
    \begin{table}[h]
        \centering
        \caption{Compute node used for the experiments. Jobs were scheduled with slurm and ran inside
        a Singularity container (PyTorch NGC 26.03). Each job was allocated a single GPU.}
        \label{tab:hardware}
        \small
        \begin{tabular}{ll}
            \toprule
            & \textbf{Resources}                      \\
            \midrule
            GPU         & NVIDIA A100-SXM4, 40\,GB HBM2           \\
            CPU         & 2$\times$ AMD EPYC 7742 (64 cores each) \\
            System RAM  & 1.0\,TB                                 \\
            OS / driver & Ubuntu (Linux 6.8), CUDA driver 580     \\
            \bottomrule
        \end{tabular}
    \end{table}

    \subsection{Data}
    We provide an overview of all datasets including references and respective licenses in \cref{tab:data-licenses}.
    \begin{table}[h]
        \centering
        \caption{Dataset Licenses.}
        \label{tab:data-licenses}
        \begin{tabular}{lll}
            \toprule
            Dataset      & Reference                                    & License                                        \\ \midrule
            CIFAR-10     & \citep{krizhevsky2009cifar}                  & No formal license                              \\
            CIFAR-10-C   & \citep{hendrycksBenchmarkingCorruptions2019} & Creative Commons Attribution 4.0 International \\
            CIFAR-10-H   & \citep{petersonCIFAR10H2019}                 & Creative Commons Attribution 4.0 International \\
            MedMNIST     & \citep{yangMedMNIST2021}                     & Creative Commons Attribution 4.0 International \\
            MedMNIST-C   & \citep{diSalvoMedMNISTC2024}                 & Creative Commons Attribution 4.0 International \\
            CIFAR-100    & \citep{krizhevsky2009cifar}                  & No formal license                              \\
            TinyImageNet & \citep{le2015tiny}                           & Non-commercial research/educational use only   \\
            MNIST        & \citep{lecunMNIST1998}                       & MIT License                                    \\
            SVHN         & \citep{netzer2011svhn}                       & Non-commercial use only                        \\
            Textures     & \citep{cimpoiTextures2014}                   & Research use only                              \\
            Places365    & \citep{zhouPlaces3652018}                    & Non-commercial research and educational use    \\
            \bottomrule
        \end{tabular}
    \end{table}

    \subsection{Models}

    \subsubsection{ResNet18} We use the ResNet18 implementation, which has been adapted to better support CIFAR-10\footnote{\url{https://github.com/kuangliu/pytorch-cifar}}. Specifically, the kernel size of the first convolutional layer has been adapted from $7\times 7$ to $3\times 3$ and the first max-pooling layer is removed to better support the size $32\times 32$ images.

    \paragraph{Hyperparameters.}
    \cref{tab:hyperparameters} shows the hyperparameters for the ResNet18. All models were trained without early stopping and the checkpoint with the lowest validation loss is used.

    \begin{table}[h]
        \centering
        \caption{Hyperparameters for ResNet18.}
        \label{tab:hyperparameters}
        \resizebox{\textwidth}{!}{
            \begin{tabular}{lccccccccc}
                \toprule
                Dataset    & Epochs & Batch Size & Val. Split & Optim. & LR    & Momentum & Weight Dec. & Scheduler            \\ \midrule
                CIFAR-10   & 200    & 128        & 0.1        & SGD    & 0.1   & 0.9      & 5e-4        & Cosine               \\
                BloodMNIST & 100    & 128        & 0.1        & Adam   & 0.001 & 0.0      & 0.0         & Factor 0.1 at 50, 75 \\
                PathMNIST  & 100    & 128        & 0.1        & Adam   & 0.001 & 0.0      & 0.0         & Factor 0.1 at 50, 75 \\
                \bottomrule
            \end{tabular}
        }
    \end{table}

    \subsection{Methods}
    The implementations of the credal prediction methods rely heavily on the \texttt{probly} library\footnote{\url{https://github.com/pwhofman/probly}}. All credal methods use the \texttt{probly} implementation, except for CreDro \citep{wangDro2026}. The base models for all methods are trained using the same hyperparameters listed in \cref{tab:hyperparameters}.

    \subsubsection{Credal Ensembling}
    This method relies on ensembling neural networks \citep{nguyen2025credal}. We use $10$ ensemble members. Given an input, the probabilistic outputs of the different ensemble members are collected. From these, a representative distribution is computed based on some distance, the outputs are ordered by their distance to the representative distribution, and an $\alpha$ portion of the distributions is discarded. We set $\alpha=0.0$ as previous works \citep{lohrCredal2025,hofmanEfficientCredalPrediction2026} showed this performs best on \ac{ood} detection. The credal set is obtained by taking the convex hull of the remaining ($1-\alpha$) portion of distributions. The representative distribution is used when a single distribution has to be predicted.

    \subsubsection{Credal Bayesian Neural Networks}
    This method is constructed by ensembling Bayesian neural networks with different prior distributions. The original paper \citep{caprio2023credal} does not list parameters of the priors, hence we use the ones suggested by \texttt{probly} --- that is, we use $10$ Bayesian neural networks with zero mean Gaussian priors that have standard deviations: 0.1, 0.15, 0.2, 0.3, 0.45, 0.7, 1.0, 1.5, 2.0, and 3.0. From each of the posteriors of the resulting $10$ Bayesian neural networks, $20$ samples are taken, which are averaged, resulting in $10$ conditional distributions. The convex hull of these distributions generates the credal set. The mean of the $10$ (conditional) distributions are then used when a single predictive distribution is required.

    \subsubsection{Credal Wrapper}
    This method is an ensembling method, which generates lower and upper probabilities based on the minimum and maximum class-wise probabilities across the ensemble members \citep{wangCredalWrapper2025}. We use $10$ ensemble members. These probability intervals are then turned into interval-based credal sets. When a single distribution is required, such as in \cref{sec:shift-set-size}, the intersection probability \citep{cuzzolinIntersectionProbability2022} is computed. In the risk-averse experiments, our decision rule is used.

    \subsubsection{Credal Distributionally Robust Optimization}
    The credal distributionally robust optimization method [CreDro] \citep{wangDro2026} has been faithfully re-implemented based on the original implementation\footnote{\url{https://github.com/Kaizheng-WANG/Learning-Credal-Ensembles-via-Distributionally-Robust-Optimization}}. The method relies on ensembling models that are trained with a loss function that optimizes based on only the $\delta_i$ portion of worst losses in a batch, where $\delta_i = \frac{1-\delta_G}{M - 1}(i-1) + \delta_G$ is varied across members. We use $M = 10$ members and $\delta_G = 0.5$, as this was shown to be the best-performing value in the original paper. The lower and upper probabilities are based on the minimum and maximum class-wise probabilities across the ensemble members, and the probability intervals are turned into interval-based credal sets. This interval-based credal set may in turn be aggregated into the intersection probability when a single distribution is required.

    \subsubsection{Credal Relative Likelihood}
    This method also builds an ensemble using a specific initialization technique to promote diversity and an early stopping technique to reach the relative likelihood threshold $\alpha$ \citep{lohrCredal2025}. We use $10$ ensemble members and $\alpha=0.95$, as this was shown to perform best on \ac{ood} detection in the original paper. In case a single distribution has to be predicted, the MLE prediction is used.

    \subsubsection{Efficient Credal Prediction}
    This method uses a base predictor and constructs probability intervals by observing how far logits can be perturbed while the resulting model stays within the relative likelihood budget $\alpha$ \citep{hofmanEfficientCredalPrediction2026}. We use $\alpha=0.95$ as this was shown to be the best performing setting for \ac{ood} detection. The resulting logit bounds are converted into probability intervals from which an interval-based credal set is constructed. The MLE predictions are used when a single distribution is required.

    \subsubsection{SQwash}
    This method makes the training objective itself risk-averse by replacing the mean of the mini-batch losses with their superquantile, that is, the empirical \ac{cvar} of the batch losses \citep{laguelSuperquantileBasedLearning2022}. In each training step, only the worst $\tau$-fraction of the cross-entropy losses within the batch is averaged. We use the implementation of the loss provided by the authors\footnote{\url{https://github.com/krishnap25/sqwash}}. In contrast to the credal methods, the result is a single network that outputs precise conditional distributions.

    \subsubsection{AdaCVaR}
    This method minimizes the \ac{cvar} of the loss distribution through adaptive sampling \citep{curiAdaptiveSampling2020}. Training is framed as a game between the learner and an adversarial sampler in which an Exp3-type bandit algorithm maintains weights per instance approximating a $k$-DPP over the $k = \lceil \tau N \rceil$ training instances with the largest losses. Then, in each step, a subset of size $k$ is sampled accordingly, and a mini-batch is drawn uniformly from this subset. The learner minimizes the mean cross-entropy on these batches and the sampler weights are updated on the observed losses. The method has been faithfully re-implemented based on the original implementation\footnote{\url{https://github.com/sebascuri/adacvar}}, using the sampler learning rate $\eta = \sqrt{\log(1/\tau)/(\tau T)}$ suggested in the original paper, where $T$ denotes the total number of training steps. Similar to SQwash, the result is a single network with precise conditional distributions as outputs.

    \clearpage

    \section{Experimental Details}\label{app:experimental}

    \subsection{Distribution Shift}\label{app:exp-distribution-shift}
    We use a ResNet18 trained on the BloodMNIST and CIFAR-10 datasets. \cref{tab:hyperparameters} shows the hyperparameters. The shifted test sets are taken from the corresponding corruption benchmarks, MedMNIST-C for BloodMNIST and CIFAR-10-C for CIFAR-10, providing $11$ and $15$ corruption types, respectively, each at five severity levels. Severity $0$ refers to the clean test split of the respective dataset. The MedMNIST-C images ship at a resolution of $224 \times 224$ and are downsampled to the native $28 \times 28$ resolution the models are trained on, after which the same (test) transformations as for the clean data are applied. The CIFAR-10-C images are already at native resolution. For every corrupted test set we compute the credal set size, that is, the width of the class-wise probability intervals averaged over classes and instances, where for the methods producing a convex hull the intervals are given by the lower and upper envelope of that hull. Within a seed the set size is first averaged over the corruption types of a severity level and the curves show the mean and standard deviation over three seeds. In addition, we report the accuracy of the regular point prediction made by the respective method, which serves as a reference for the degradation of the underlying model.

    \subsection{Out-of-Distribution Detection}\label{app:exp-ood}
    We use a ResNet18 trained on CIFAR-10 and one trained on BloodMNIST for the presented \ac{ood} detection tasks. \cref{tab:hyperparameters} shows the hyperparameters. The task is framed as a binary classification task where the epistemic uncertainty is used to discriminate between \ac{id} and \ac{ood} instances, correspondingly the \ac{auroc} is computed. All \ac{ood} detection experiments are run three times and the mean and standard deviation are reported. The test splits of the respective datasets are used and both \ac{id} and \ac{ood} instances receive the same (test) transformations, except for the grayscale images, which receive an additional transformation duplicating the grayscale values into three channels. In order to make the binary classification task balanced, the minimum of the number of samples of both datasets is taken. The CIFAR-10 experiments follow the setup from \citet{zhangOpenOOD2024}.

    \subsubsection{MedMNIST}
    For the MedMNIST collection, we use BloodMNIST as the \ac{id} dataset, which is a collection of RGB images of blood cells. Following the setup used in CIFAR-10, we use two near-\ac{ood} datasets: PathMNIST and TissueMNIST. Both datasets show cells or tissue under a microscope, meaning low level image statistics should be comparable. The DermaMNIST, RetinaMNIST, OCTMNIST, and BreastMNIST are used as the far-\ac{ood} datasets.

    \subsection{Coverage versus Efficiency}\label{app:exp-coverage-efficiency}
    We use the ResNet18 models trained on CIFAR-10 (\cref{tab:hyperparameters}) and treat the human soft labels of CIFAR-10-H \citep{petersonCIFAR10H2019} as approximations of the ground-truth conditional distributions on the test set. Coverage is the fraction of test instances whose ground-truth distribution lies in the predicted credal set. Membership is tested exactly for each representation, by containment in the class-wise intervals for interval-based sets and by linear programming for convex sets. Efficiency is one minus the mean width of the credal set, so that higher values indicate tighter sets. Both are properties of the credal set alone, involving no decision rule or loss. For the methods parameterized by a plausibility level, the curve is traced over $\alpha \in \{0, 0.2, 0.4, 0.6, 0.8, 0.9, 0.95, 1\}$, whereas the ensemble-based methods yield a single point. All results are averaged over $3$ seeds. Note, at $\alpha = 0$ the log-likelihood threshold $\log \alpha$ is not defined, so this value is understood as the limiting case in which every hypothesis is deemed plausible and the credal set is the whole simplex.

    \subsection{Risk-Aversion under Distribution Shift}\label{app:exp-risk-aversion-shift}
    The experiments on risk-averse prediction under distribution shift are performed using the ResNet18 base model for each method. For all datasets, the training data and the clean test data are taken from MedMNIST. The models are trained on $90\%$ of the training data, with the remaining $10\%$ held out as validation data. The risk-averse baselines are trained with the same $\tau$-level as the \ac{cvar} under which they are evaluated.

    For the MLE and the risk-averse baselines, the predictive distribution produced by the model is taken as the prediction itself. When the MLE model is combined with the \ac{cvar} rule \cref{eq:cvar-distribution}, the threshold $v$ is first determined on the (clean) validation data, and the prediction are then derived as described in \cref{sec:minimizing-cvar}.
    The credal methods, including CreWare, follow the same scheme: the threshold $v$ is computed on the validation data as in \cref{alg:threshold} and predictions are obtained from the \ac{cvar}-minimax decision rule \cref{eq:cvar-distribution-set}.
    In both bases, the decision rule is applied with the same $\tau$ as the \ac{cvar} under which the predictions are evaluated.
    As a reference, we also include a uniform prediction baseline, which predicts a uniform distribution regardless of the input, in order to detect trivial predictions from a method.

    The corrupted test data come from MedMNIST-C, which specifies a set of modality-specific corruptions for each dataset at five levels of severity. For a given severity level, all methods are evaluated by applying each of these corruptions and averaging over corruption types, resulting in a single loss value per severity level. This aligns with the setup in \citep{ovadiaDatasetShift2019}.
    The performance of a method is then measured in terms of the \ac{cvar} at $\tau$-level of the log-loss and the Brier-loss, computed separately on each corrupted test set prior to averaging.

    We report the mean and standard deviation over three seeds.

    \subsection{Cost-Sensitive Decision Making}\label{app:cost-sensitive}

    \paragraph{Setup.}
    The cost matrix is given in \cref{tab:triage-cost}. Sending a septic patient home is the (uniquely) worst action for its condition and hence the most catastrophic decision, following the critical-mistake notion of \citet{kiyani2025decision}.
    Mistreating a flu patient by sending them home or to ICU costs $4$, so neither is catastrophic. Precaution costs $6$--$8$ on a healthy patient, so admitting everyone to ICU realizes a \ac{cvar} of exactly $8$ at every $\tau$ considered.

    \begin{table}[h]
        \centering
        \caption{Cost matrix of the triage problem (rows: true condition, columns: action).}
        \label{tab:triage-cost}
        \begin{tabular}{@{}lccc@{}}
            \toprule
            & \textbf{Send home} & \textbf{Treat flu} & \textbf{ICU} \\ \midrule
            Healthy & $0$                & $6$                & $8$          \\
            Flu     & $4$                & $1$                & $4$          \\
            Sepsis  & $20$               & $14$               & $2$          \\ \bottomrule
        \end{tabular}
    \end{table}

    \paragraph{Data.}
    Instances have two informative features, drawn from the Gaussian mixture in \cref{tab:triage-dgp}, and six additional features drawn i.i.d.\ from $\mathcal{N}(0,1)$ representing noise. The class prior is $0.55$ / $0.25$ / $0.20$ (healthy / flu / sepsis). Typical healthy, flu, and septic patients form three separated clusters. A quarter of all septic patients are masked, that is, they present with almost healthy features and overlap the typical healthy subgroup. In addition, $15\%$ of the healthy and $15\%$ of the septic patients belong to a subgroup with an atypical presentation, far from all other components. The training data are scarce for the atypical septic patients, of which only $5\%$ are available in training ($\approx 6$ of $120$ instances). All other components, including the healthy patients of the subgroup, are used in full, so the training data show the atypical presentation almost exclusively as healthy. Validation and test data keep the natural frequencies. The splits contain $4000$ / $2000$ / $8000$ instances.

    \begin{table}[h]
        \centering
        \caption{Components of the data-generating process (informative coordinates). Shares refer to the natural frequencies of the validation and test data. Within each condition they sum to $100\%$, as do the shares of all patients. The last column gives the fraction of each component that is used in training.}
        \label{tab:triage-dgp}
        \begin{tabular}{@{}llccccc@{}}
            \toprule
            & & & & \multicolumn{2}{c}{\textbf{Share of}} & \textbf{Used in} \\ \cmidrule(lr){5-6}
            \textbf{Condition} & \textbf{Presentation}      & \textbf{Mean} & \textbf{Std.} & \textbf{condition} & \textbf{all patients} & \textbf{training}    \\ \midrule
            Healthy            & Typical                    & $(0, 0)$      & $1.0$         & $85\%$             & $46.75\%$             & $100\%$           \\
            & Atypical (subgroup)        & $(-6, 4)$     & $0.8$         & $15\%$             & $8.25\%$              & $100\%$           \\ \addlinespace
            Flu                & Typical                    & $(3, 0)$      & $1.0$         & $100\%$            & $25.00\%$             & $100\%$           \\ \addlinespace
            Sepsis             & Typical                    & $(3, 3)$      & $0.8$         & $60\%$             & $12.00\%$             & $100\%$           \\
            & Masked (resembles healthy) & $(-1.8, 1.8)$ & $0.5$         & $25\%$             & $5.00\%$              & $100\%$           \\
            & Atypical (subgroup)        & $(-9, 7)$     & $1.5$         & $15\%$             & $3.00\%$              & \textbf{5\%}      \\ \bottomrule
        \end{tabular}
    \end{table}

    \paragraph{Agents.}
    All point predictors share the same backbone which is an MLP with two hidden layers of width $64$, trained with Adam (learning rate $10^{-3}$, batch size $128$, $150$ epochs) on standardized inputs. The MLE decides by the risk-neutral best response. SQwash \citep{laguelSuperquantileBasedLearning2022} (smoothed reducer, smoothing $1.0$) and AdaCVaR \citep{curiAdaptiveSampling2020} are trained at a tail level equal to the evaluation $\tau$. CreWra uses $10$ ensemble members. Our predictor is fit as in \cref{app:practical-implementation} on the standardized inputs, with $200$ neighbors, bandwidth $0.3$ times the median $200$th-neighbor distance, and $\alpha = 0.95$. Thresholds are calibrated per method, seed, and $\tau$ on the validation split. We report the empirical \ac{cvar} \cref{eq:cvar-empirical} of the realized test costs, the mean cost, and the fraction of septic patients sent home, over $10$ seeds.

    \paragraph{Decision rule.}
    For a finite action set $\mathcal{A}$ and a cost function $c : \cY \times \mathcal{A} \to \bR_{\geq 0}$, the rule \cref{eq:cvar-distribution-set} becomes
    \begin{equation}\label{eq:cvar-action-set}
        \tilde{a}(\vx) = \argmin_{a \in \mathcal{A}}\ \sup_{q \in \cQ_{\vx}} \bE_{Y \sim q}\left[(c(Y, a) - v)_+\right],
    \end{equation}
    which is exact by enumeration. $v$ is calibrated by \cref{alg:threshold} with realized costs in place of losses.

    \subsection{Risk-Averse Reinforcement Learning}\label{app:rl-details}

    \paragraph{Setup.}
    The $11 \times 13$ grid in \cref{fig:driving} has $64$ free cells, six of which are hazard cells forming three rows by two columns on the direct road. The car moves deterministically in four directions and stays put at walls. Entering a hazard cell crashes with probability $p = 0.1$. Every step costs $r_s = -1$, a crash ends the episode with $r_c = -8$ in place of the step cost, and reaching the goal pays $r_g = +40$ in addition to it. The direct route ($8$ steps, three hazard cells) crashes with probability $1-(1-p)^3 \approx 0.27$ and attains an expected reward of $20.4$, whereas the safe detour ($26$ steps) always returns a reward of $14$.

    \paragraph{Data and predictions.}
    Every non-hazard cell is observed $n = 1000$ times without a crash, every hazard cell $k$ times with $\mathrm{Binomial}(k, p)$ crashes, for $k \in \{1, 2, 3, 4, 6, 8, 13, 20, 30, 40, 50\}$. From the counts $(k_s, c_s)$ of a cell $s$, the point estimate is $\hat p_s = c_s / k_s$, and $\cQ_s = [\lowp_s, \uppp_s]$ is the $\alpha$-cut of the Bernoulli relative likelihood,
    \begin{equation*}
        \cQ_s = \left\{p \in [0,1] : \frac{p^{c_s} (1-p)^{k_s - c_s}}{\hat p_s^{\,c_s} (1-\hat p_s)^{k_s - c_s}} \geq \alpha \right\},
    \end{equation*}
    with $\alpha = 0.15$. Without observed crashes, $\hat p_s = 0$ while $\uppp_s = 1 - \alpha^{1/k_s}$, that is, $0.85$ after one, $0.21$ after eight, and $0.09$ after twenty observations.

    The non-hazard cells carry residual epistemic uncertainty as well, with $\uppp_s = 1 - \alpha^{1/n} \approx 0.0019$, which accumulates to a worst-case crash probability of $0.046$ over the $25$ intermediate cells of the detour. The detour is preferred under the worst case only if these cells are sampled sufficiently often.

    \paragraph{Agents.}
    MLE maximizes the expected return under $\hat p$ by value iteration on the Bellman equation \citep{sutton1998reinforcement}
    \begin{equation}\label{eq:bellman-rl}
        Q(s,a)=
        \begin{cases}
            r_s + r_g, & s'=g,\\[2pt]
            p_{s'}\,r_c+(1-p_{s'})(r_s+ V_{s'}), & s'\neq g,
        \end{cases}
    \end{equation}
    with $V_s = \max_a Q(s,a)$. Credal\plus rule applies \cref{eq:cvar-distribution-set} to the negative return $L^\pi = -G^\pi$ of a rollout. For a fixed threshold $v$, the rule selects the policy that minimizes the worst-case hinged loss $\sup_{q} \bE_q[(L^\pi - v)_+]$ over the cell-wise credal sets. In the supervised setting, $v$ is calibrated on the realized losses of hold-out data according to \cref{alg:threshold}. An agent that plans offline has no realized returns of a candidate policy at its disposal, so $v$ is instead optimized jointly with the policy under the worst-case model,
    \begin{equation}\label{eq:cvar-rl}
        \hat\pi=\argmin_\pi\ \min_{v}\ \Big\{ v + \frac{1}{\tau} \sup_{q(\cdot)\,\in\,\cQ_{(\cdot)}} \bE_q\big[(L^\pi - v)_+\big] \Big\},
    \end{equation}
    with $\alpha = 0.15$ and $\tau = 0.1$. The plan thus depends on the observed counts only. By the variational form \cref{eq:variational-cvar}, the objective is an upper bound on the worst-case risk $\sup_q \cvar_\tau(L^\pi)$, as the minimum over $v$ is taken outside the supremum. The supremum itself is not obtained at the upper bounds $\uppp_s$ in general. A crash ends the episode at a loss that is lower than that of any later crash, hence a larger crash probability in one cell does not stochastically increase $L^\pi$. However, the hinged expectation is linear in the crash probability of each cell, and the credal sets are cell-wise, so the supremum is attained at an endpoint of each $\cQ_s$, which is selected by the recursion below. MLE\plus rule replaces the credal sets by $\hat p_s$. It is the exact optimum, under the estimated model, of the \ac{cvar} objective that risk-averse policy-gradient methods such as CeSoR \citep{greenbergRiskAverseRL2022} optimize by sampling.

    \paragraph{Budget-augmented value iteration.}
    In \cref{eq:cvar-rl}, the inner problem at fixed $v$ is solved on the state $(x, s)$ augmented by the step count $s$, which equals the accumulated loss \citep{bauerle2011markov}. As $s$ grows by one per step, a single backward pass over $s = s_{\max}-1, \ldots, 0$ suffices, with $W_{s_{\max}} \equiv \infty$ and
    \begin{equation*}
        W_s(x) = \min_{a} \begin{cases}
                              (s + 1 - r_g - v)_+, & x' = g, \\[2pt]
                \max\limits_{q \in \{\lowp_{x'}, \uppp_{x'}\}} q\,(s - r_c - v)_+ + (1 - q)\, W_{s+1}(x'), & x' \neq g,
        \end{cases}
    \end{equation*}
    where $x'$ is the successor of $x$ under $a$. The maximum picks $\uppp_{x'}$ where a crash in $x'$ costs more than continuing and $\lowp_{x'}$ otherwise. We take $s_{\max} = 64$, the number of free cells, which is exact since revisiting a cell is never optimal. The outer minimization scans $v$ in increasing order over the attainable losses, which are integers, and evaluates the objective of \cref{eq:cvar-rl} for the route planned at each $v$. As this objective is at least $v$, the scan stops once $v$ exceeds the best value found. Routes that tie in the hinged loss are ordered by their expected loss.

    \paragraph{Evaluation.}
    The evaluation is done using $150$ seeds, each drawing a dataset from which every agent plans a route that is then driven $1000$ times under the true crash probabilities.
    The crash rate and mean reward are averaged over seeds, with bands of one standard deviation, whereas the CVaR reward is the population $\cvar_{0.1}$ of all pooled returns, the marginal \ac{cvar} of \cref{app:variants}, and therefore has no band. Averaged over $k$, the MLE attains a mean reward of $19.1$ at a crash rate of $0.27$, below the $20.4$ of the direct route. For small $k$, it avoids hazard cells with an observed crash, which leads it through a fourth hazard cell on a $10$-step route in some seeds and onto the detour in a few others.

    \clearpage

    \section{Additional Results}\label{app:additional-results}

    \subsection{Credal Sets under Distribution Shifts}\label{app:distributional-shifts}
    We provide additional results with the credal predictors trained on the CIFAR-10 dataset using the corruptions from CIFAR-10-C. We use all corruptions and severities and report the set size and accuracy averaged over the corruptions for each severity level. We show the results in \cref{fig:dist-shift-cifar}. CreWare adapts suitably to the distribution shifts, increasing its sets accordingly with the diminishing accuracy.

    \begin{figure}[h]
        \centering
        \includegraphics[width=0.5\linewidth]{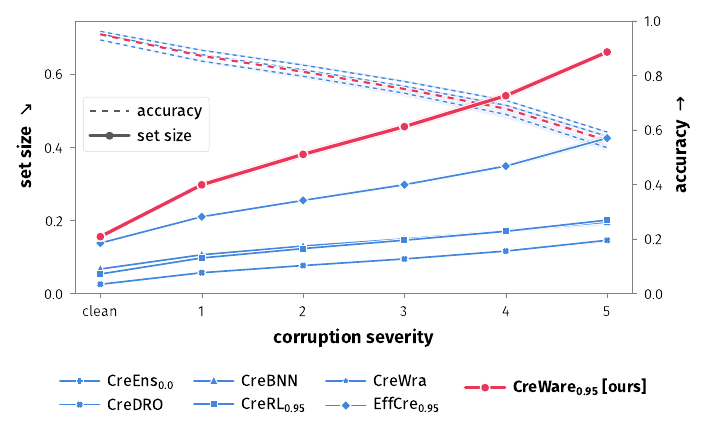}
        \caption{Distribution shift on CIFAR-10. The mean and standard deviation of the set size and accuracy are reported. The sets of CreWare increase accordingly with the diminishing accuracy, while the other sets increase less in size.}
        \label{fig:dist-shift-cifar}
    \end{figure}

    \subsection{Coverage-Efficiency of Credal Predictor}\label{app:coverage-efficiency}
    We evaluate the coverage and efficiency of our novel credal predictor on the CIFAR-10 dataset, using the human annotations provided by CIFAR-10-H \citep{petersonCIFAR10H2019}, which we consider as (approximations of) the ground-truth distributions. \cref{fig:cov_eff} shows the results.

    \begin{figure}[h]
        \centering
        \includegraphics[width=0.4\linewidth]{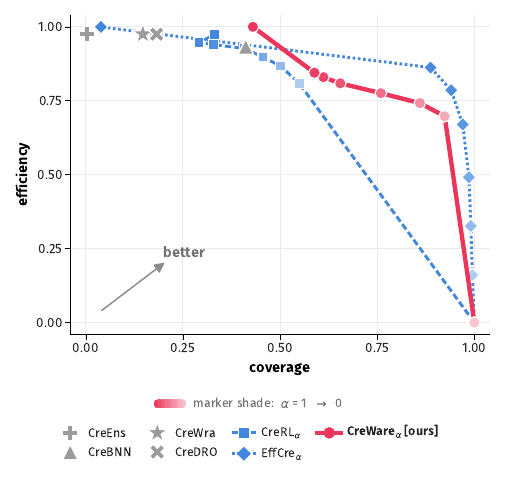}
        \caption{Coverage versus efficiency for CIFAR-10-H. The CreWare method performs competitively, staying in the high coverage region while having a good efficiency.}
        \label{fig:cov_eff}
    \end{figure}

    \subsection{Out-of-distribution Detection}\label{app:ood-results}
    In addition to the experiments on the CIFAR-10 benchmark, we perform further \ac{ood} detection experiments in order to evaluate our credal predictor.

    \cref{tab:ood-full} first completes \cref{tab:ood} from the main text with the far-\ac{ood} datasets MNIST \citep{lecunMNIST1998} and Textures \citep{cimpoiTextures2014}.

    \begin{table*}[h]
        \centering
        \caption{\Acl*{ood} detection with CIFAR-10 on all near-\acs*{ood} and far-\acs*{ood} datasets. The mean and standard deviation of the \acs*{auroc} over three runs are reported. The best method is shown in \textbf{bold}. Subscripts denote the method-specific parameter $\alpha$ controlling the set size.}
        \label{tab:ood-full}
        \resizebox{\textwidth}{!}{
            \begin{tabular}{@{}lcc|cccc@{}}
                \toprule
                & \multicolumn{2}{c}{\textbf{Near-OOD}} & \multicolumn{4}{c}{\textbf{Far-OOD}} \\
                \textbf{Method}                         & \textbf{CIFAR-100}                      & \textbf{TinyImageNet}                   & \textbf{MNIST}                          & \textbf{SVHN} & \textbf{Textures} & \textbf{Places365} \\ \midrule
                $\text{CreEns}_{0.0}$                   & $0.917 \scriptstyle \pm 0.000$          & $0.911 \scriptstyle \pm 0.001$ & $0.961 \scriptstyle \pm 0.002$ & $0.970 \scriptstyle \pm 0.003$ & $0.948 \scriptstyle \pm 0.000$ & $0.919 \scriptstyle \pm 0.003$ \\
                CreBNN                                  & $0.867 \scriptstyle \pm 0.007$          & $0.866 \scriptstyle \pm 0.006$          & $0.912 \scriptstyle \pm 0.024$ & $0.892 \scriptstyle \pm 0.016$ & $0.894 \scriptstyle \pm 0.011$ & $0.872 \scriptstyle \pm 0.006$ \\
                CreWra                                  & $0.918 \scriptstyle \pm 0.000$          & $0.912 \scriptstyle \pm 0.001$          & $0.961 \scriptstyle \pm 0.002$ & $0.971 \scriptstyle \pm 0.003$ & $0.950 \scriptstyle \pm 0.001$ & $0.920 \scriptstyle \pm 0.003$ \\
                CreDro                                  & $0.918 \scriptstyle \pm 0.001$          & $0.911 \scriptstyle \pm 0.001$          & $0.965 \scriptstyle \pm 0.003$ & $0.972 \scriptstyle \pm 0.001$ & $0.950 \scriptstyle \pm 0.001$ & $0.920 \scriptstyle \pm 0.003$ \\
                $\text{CreRL}_{0.95}$                   & $0.913 \scriptstyle \pm 0.003$          & $0.908 \scriptstyle \pm 0.001$ & $0.953 \scriptstyle \pm 0.013$ & $0.969 \scriptstyle \pm 0.007$ & $0.944 \scriptstyle \pm 0.001$ & $0.917 \scriptstyle \pm 0.003$ \\
                $\text{EffCre}_{0.95}$                  & $0.859 \scriptstyle \pm 0.001$          & $0.857 \scriptstyle \pm 0.001$ & $0.905 \scriptstyle \pm 0.006$ & $0.911 \scriptstyle \pm 0.006$ & $0.890 \scriptstyle \pm 0.001$ & $0.871 \scriptstyle \pm 0.003$ \\
                $\text{CreWare}_{0.95}$ [\textbf{ours}] & $\mathbf{0.931 \scriptstyle \pm 0.002}$ & $\mathbf{0.929 \scriptstyle \pm 0.001}$ & $\mathbf{0.980 \scriptstyle \pm 0.005}$ & $\mathbf{0.980 \scriptstyle \pm 0.005}$ & $\mathbf{0.970 \scriptstyle \pm 0.004}$ & $\mathbf{0.941 \scriptstyle \pm 0.002}$ \\ \bottomrule
            \end{tabular}
        }
    \end{table*}

    \subsubsection{Entropy-Based versus Set-Size-Based Measures}\label{app:ood-measures}
    The measure of epistemic uncertainty \cref{eq:epistemic-uncertainty} by \citet{abellanUncertaintyMeasure2006} has recently been criticized for not fulfilling certain desirable properties \citep{huellermeierCredalUncertainty2022}. In particular, the measure is not shift invariant and, consequently, the epistemic uncertainty of a credal set will decrease if the credal set is shifted towards the center of the probability simplex (the uniform distribution) even though the probability intervals keep the same length. Therefore, we additionally perform experiments with the same setup, but using the set size of the credal set \cref{eq:set-size} as a discriminator for each method. \cref{tab:ood-set-size} shows the results. In addition to outperforming every method with the entropy-based measure of uncertainty, CreWare also outperforms the methods across near-\ac{ood} and far-\ac{ood} when the set size is used.

    \begin{table*}[h]
        \centering
        \caption{\Acl*{ood} detection with CIFAR-10 using the set size as discriminator. The mean and standard deviation of the \acs*{auroc} over three runs are reported. The best method is shown in \textbf{bold}. CreWare outperforms other methods across near-\acs*{ood} and far-\acs*{ood} datasets.}
        \label{tab:ood-set-size}
        \resizebox{\textwidth}{!}{
            \begin{tabular}{@{}lcc|cccc@{}}
                \toprule
                & \multicolumn{2}{c}{\textbf{Near-OOD}} & \multicolumn{4}{c}{\textbf{Far-OOD}} \\
                \textbf{Method}                         & \textbf{CIFAR-100}                      & \textbf{TinyImageNet}                   & \textbf{MNIST}                          & \textbf{SVHN} & \textbf{Textures} & \textbf{Places365} \\ \midrule
                $\text{CreEns}_{0.0}$                   & $0.911 \scriptstyle \pm 0.001$          & $0.903 \scriptstyle \pm 0.001$ & $0.946 \scriptstyle \pm 0.002$ & $0.955 \scriptstyle \pm 0.004$ & $0.940 \scriptstyle \pm 0.002$ & $0.909 \scriptstyle \pm 0.003$ \\
                CreBNN                                  & $0.872 \scriptstyle \pm 0.005$          & $0.869 \scriptstyle \pm 0.004$          & $0.905 \scriptstyle \pm 0.010$ & $0.907 \scriptstyle \pm 0.011$ & $0.901 \scriptstyle \pm 0.010$ & $0.872 \scriptstyle \pm 0.005$ \\
                CreWra                                  & $0.911 \scriptstyle \pm 0.001$          & $0.903 \scriptstyle \pm 0.001$          & $0.946 \scriptstyle \pm 0.002$ & $0.955 \scriptstyle \pm 0.004$ & $0.940 \scriptstyle \pm 0.002$ & $0.909 \scriptstyle \pm 0.003$ \\
                CreDro                                  & $0.911 \scriptstyle \pm 0.001$          & $0.903 \scriptstyle \pm 0.001$          & $0.949 \scriptstyle \pm 0.002$ & $0.960 \scriptstyle \pm 0.004$ & $0.941 \scriptstyle \pm 0.002$ & $0.909 \scriptstyle \pm 0.003$ \\
                $\text{CreRL}_{0.95}$                   & $0.909 \scriptstyle \pm 0.003$          & $0.901 \scriptstyle \pm 0.000$ & $0.935 \scriptstyle \pm 0.018$ & $0.960 \scriptstyle \pm 0.008$ & $0.934 \scriptstyle \pm 0.003$ & $0.905 \scriptstyle \pm 0.002$ \\
                $\text{EffCre}_{0.95}$                  & $0.879 \scriptstyle \pm 0.002$          & $0.880 \scriptstyle \pm 0.003$ & $0.945 \scriptstyle \pm 0.014$ & $0.942 \scriptstyle \pm 0.012$ & $0.921 \scriptstyle \pm 0.005$ & $0.898 \scriptstyle \pm 0.002$ \\
                $\text{CreWare}_{0.95}$ [\textbf{ours}] & $\mathbf{0.926 \scriptstyle \pm 0.002}$ & $\mathbf{0.924 \scriptstyle \pm 0.002}$ & $\mathbf{0.976 \scriptstyle \pm 0.006}$ & $\mathbf{0.976 \scriptstyle \pm 0.006}$ & $\mathbf{0.966 \scriptstyle \pm 0.004}$ & $\mathbf{0.937 \scriptstyle \pm 0.002}$ \\ \bottomrule
            \end{tabular}
        }
    \end{table*}

    \subsubsection{MedMNIST}\label{app:ood-medmnist} We perform \ac{ood} detection using datasets from the MedMNIST collection. In the medical domain in particular, it is critical that a credal predictor can differentiate \ac{id} instances from \ac{ood} instances to ensure that no (risk-averse) predictions are made on data it has not seen during training. We train a ResNet18 on BloodMNIST, and use PathMNIST and TissueMNIST as near-\ac{ood} datasets, while DermaMNIST, RetinaMNIST, OCTMNIST, and BreastMNIST serve as the far-\ac{ood} datasets. \cref{tab:ood-medmnist} shows the results. CreWare outperforms the other methods across near-\ac{ood} and far-\ac{ood} datasets. \ac{ood} detection with this suite of datasets proves to be a particularly difficult task, with many methods performing worse than random, that is, conveying higher epistemic uncertainty for \ac{id} instances than for \ac{ood} instances.

    \begin{table*}[h]
        \centering
        \caption{\Acl*{ood} detection with BloodMNIST. The mean and standard deviation of the \acs*{auroc} over three runs are reported. The best method is shown in \textbf{bold}. CreWare outperforms other methods across near-\acs*{ood} and far-\acs*{ood} datasets.}
        \label{tab:ood-medmnist}
        \resizebox{\textwidth}{!}{
            \begin{tabular}{@{}lcc|cccc@{}}
                \toprule
                & \multicolumn{2}{c}{\textbf{Near-OOD}} & \multicolumn{4}{c}{\textbf{Far-OOD}} \\
                \textbf{Method}                         & \textbf{PathMNIST}                      & \textbf{TissueMNIST}                    & \textbf{DermaMNIST}                     & \textbf{RetinaMNIST} & \textbf{OCTMNIST} & \textbf{BreastMNIST} \\ \midrule
                $\text{CreEns}_{0.0}$                   & $0.780 \scriptstyle \pm 0.011$          & $0.497 \scriptstyle \pm 0.008$ & $0.681 \scriptstyle \pm 0.016$ & $0.089 \scriptstyle \pm 0.009$ & $0.219 \scriptstyle \pm 0.014$ & $0.043 \scriptstyle \pm 0.001$ \\
                CreBNN                                  & $0.744 \scriptstyle \pm 0.043$          & $0.009 \scriptstyle \pm 0.015$          & $0.788 \scriptstyle \pm 0.015$ & $0.337 \scriptstyle \pm 0.143$ & $0.000 \scriptstyle \pm 0.001$ & $0.001 \scriptstyle \pm 0.002$ \\
                CreWra                                  & $0.773 \scriptstyle \pm 0.015$          & $0.440 \scriptstyle \pm 0.100$          & $0.637 \scriptstyle \pm 0.023$ & $0.067 \scriptstyle \pm 0.003$ & $0.180 \scriptstyle \pm 0.045$ & $0.025 \scriptstyle \pm 0.001$ \\
                CreDro                                  & $0.761 \scriptstyle \pm 0.022$          & $0.362 \scriptstyle \pm 0.065$          & $0.603 \scriptstyle \pm 0.047$ & $0.062 \scriptstyle \pm 0.002$ & $0.143 \scriptstyle \pm 0.014$ & $0.024 \scriptstyle \pm 0.000$ \\
                $\text{CreRL}_{0.95}$                   & $0.573 \scriptstyle \pm 0.021$          & $0.500 \scriptstyle \pm 0.000$ & $0.288 \scriptstyle \pm 0.006$ & $0.105 \scriptstyle \pm 0.000$ & $0.226 \scriptstyle \pm 0.000$ & $0.043 \scriptstyle \pm 0.000$ \\
                $\text{EffCre}_{0.95}$                  & $0.730 \scriptstyle \pm 0.010$          & $0.734 \scriptstyle \pm 0.094$ & $0.698 \scriptstyle \pm 0.055$ & $0.258 \scriptstyle \pm 0.053$ & $0.445 \scriptstyle \pm 0.086$ & $0.095 \scriptstyle \pm 0.008$ \\
                $\text{CreWare}_{0.95}$ [\textbf{ours}] & $\mathbf{0.906 \scriptstyle \pm 0.065}$ & $\mathbf{0.999 \scriptstyle \pm 0.001}$ & $\mathbf{0.962 \scriptstyle \pm 0.017}$ & $\mathbf{0.995 \scriptstyle \pm 0.005}$ & $\mathbf{0.998 \scriptstyle \pm 0.002}$ & $\mathbf{0.983 \scriptstyle \pm 0.014}$ \\
                \bottomrule
            \end{tabular}
        }
    \end{table*}

    \subsection{Computation Time Comparison}\label{app:time}
    In addition to adapting to distribution shifts and \ac{ood} instances, a credal predictor should be efficient if it is to be used with increasingly large models. Whether it is by training an ensemble or by modifying the backbone and its training, approaches that cannot operate post-hoc on a single pre-trained network become prohibitive as backbones grow.

    Therefore, we record the training time on CIFAR-10, as well as the inference time for producing credal sets on all test instances, across the credal baselines. The comparison thus includes the one-time training cost and the per-instance (test) cost. We refer to \cref{app:implementation} for details about the compute resources on which all measurements are run.

    \cref{fig:time} shows the recorded times in seconds, averaged over three runs. Since fitting CreWare amounts to training a single backbone followed by one encoder pass over the training data, it is an order of magnitude cheaper than the ensemble-based methods (CreEns, CreWra, CreDro, CreRL), which pay the full cost of training their members, and more than twenty times cheaper than CreBNN. Only EffCre, likewise built on a single network, trains at a comparable cost. At test time, every method except CreBNN produces its credal sets within a few seconds on the full test split.
    \begin{figure}[h]
        \centering
        \includegraphics[width=0.9\linewidth]{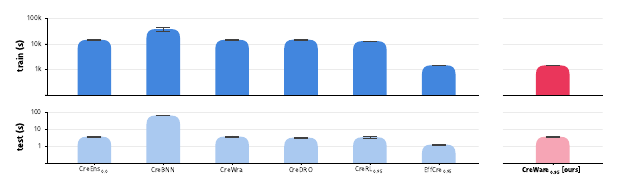}
        \caption{Train and test time in seconds (log-scale) on the full train and test splits of CIFAR-10. The mean and standard deviation over three runs are reported.}
        \label{fig:time}
    \end{figure}

    \clearpage

    \subsection{Risk-Aversion under Distribution Shift}\label{app:res-risk-aversion-shift}
    We present additional results on risk-averse prediction under distribution shifts using the log-loss and the Brier-loss including an additional dataset, and provide an ablation on the effect of the $\alpha$-parameter in our credal predictor.

    \subsubsection{Log-loss}
    First, we complete \cref{fig:risk-averse-shift} in \cref{sec:empirical-results}, with additional $\tau$-values, resulting in $\tau \in \{0.01, 0.025, 0.05, 0.1, 0.2\}$.  \cref{fig:shift_risk_pathmnist_log_loss} shows the complete results for the PathMNIST dataset. In addition, we perform the same experiment using the BloodMNIST dataset, which is also included in the MedMNIST collection. The results are shown in \cref{fig:shift_risk_bloodmnist_log_loss}.

    \begin{figure}[h]
        \centering
        \includegraphics[width=0.8\linewidth]{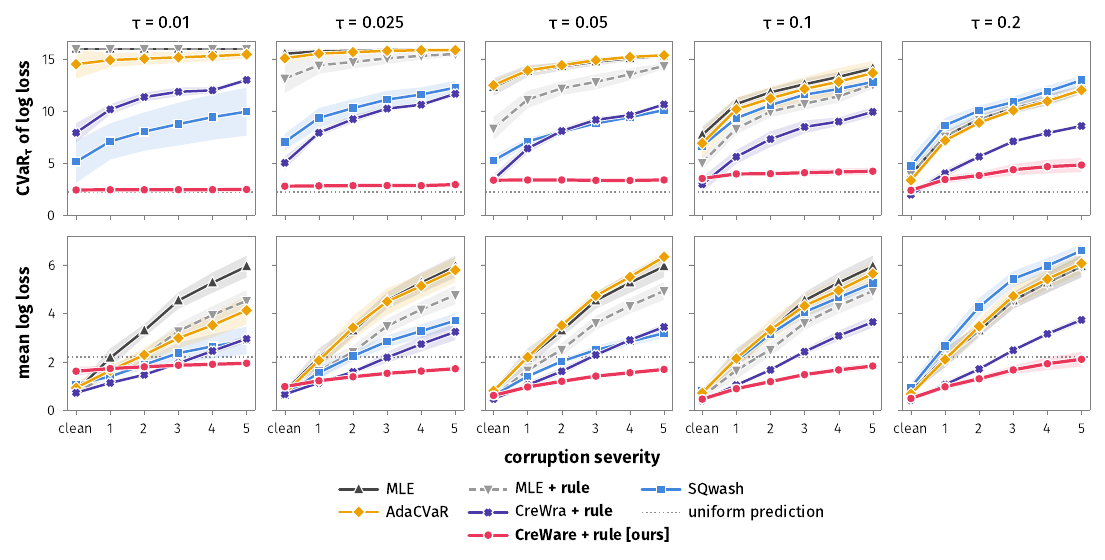}
        \caption{Risk-averse prediction under distribution shift with PathMNIST evaluated by log-loss.}
        \label{fig:shift_risk_pathmnist_log_loss}
    \end{figure}
    \begin{figure}[h]
        \centering
        \includegraphics[width=0.8\linewidth]{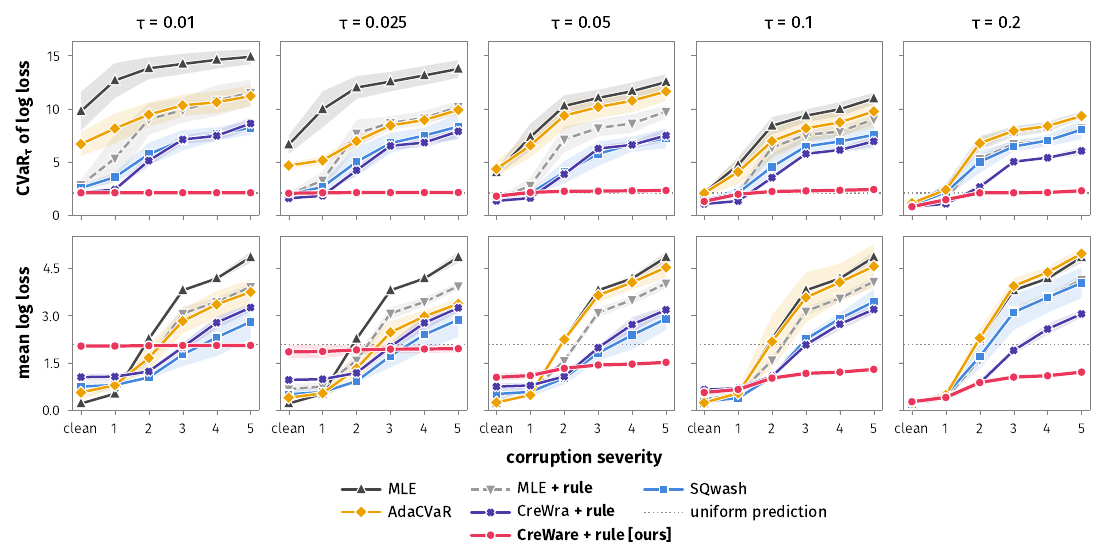}
        \caption{Risk-averse prediction under distribution shift with BloodMNIST evaluated by log-loss.}
        \label{fig:shift_risk_bloodmnist_log_loss}
    \end{figure}
    Our method outperforms or equals the risk-averse and credal baselines in terms of \ac{cvar}, while sacrificing little in terms of mean loss, and often outperforming the baselines in terms of mean loss.
    Importantly, the gap between our method and the baselines increases as the distribution shift becomes more pronounced. This is the result of our credal predictor increasing its epistemic uncertainty, and the \ac{cvar}-minimax rule responding accordingly.
    It should also be noted that, as $\tau$ increases, the difference between our method and baselines in terms of \ac{cvar} decreases, while for the mean loss it increases. This can be explained by recalling the special cases discussed in \cref{sec:risk-averse-prediction}. As $\tau$ approaches 1, the \ac{cvar} tends towards the mean loss, such that the predictions produced by our pipeline approach the ones that minimize the worst-case mean loss.
    As a result, the advantage of our method in terms of mean loss increases, while it diminishes in terms of \ac{cvar}.

    \subsubsection{Brier-loss}
    Besides the log-loss, we also evaluate our approach on another common proper scoring rule, the Brier-loss. This loss penalizes overconfident wrong predictions less extremely than the log-loss and is bounded. We provide results for the PathMNIST dataset in \cref{fig:shift_risk_pathmnist_brier}. In addition, we provide results using the BloodMNIST dataset in \cref{fig:shift_risk_bloodmnist_brier}.
    \begin{figure}[h]
        \centering
        \includegraphics[width=0.8\linewidth]{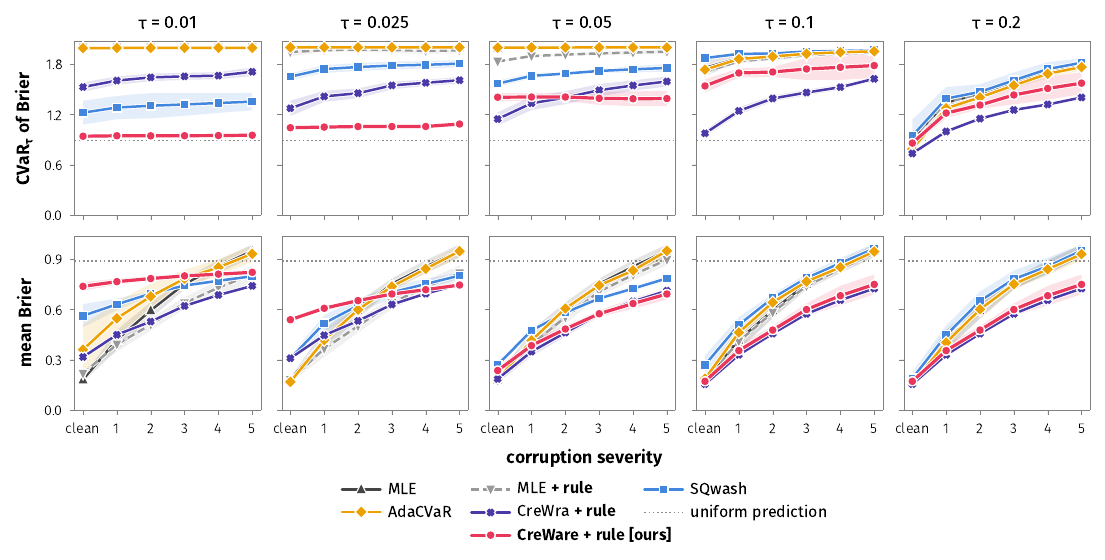}
        \caption{Risk-averse prediction under distribution shift with PathMNIST evaluated by Brier-loss.}
        \label{fig:shift_risk_pathmnist_brier}
    \end{figure}
    \begin{figure}[h]
        \centering
        \includegraphics[width=0.8\linewidth]{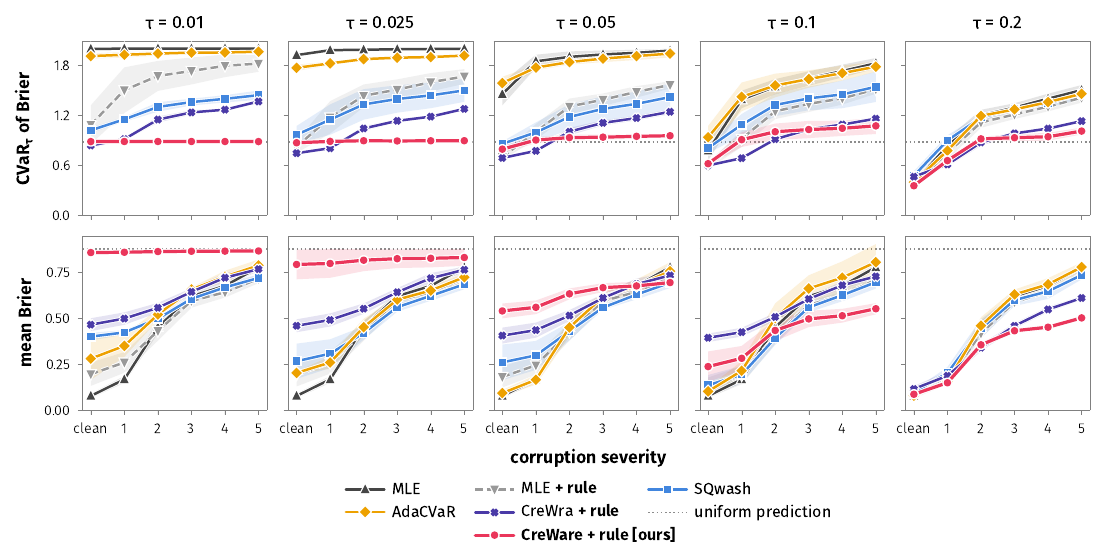}
        \caption{Risk-averse prediction under distribution shift with BloodMNIST evaluated by Brier-loss.}
        \label{fig:shift_risk_bloodmnist_brier}
    \end{figure}
    Generally speaking, the results for the Brier-loss are less extreme due to its boundedness.
    Whereas for the log-loss a single overconfident error may cause a high loss, the Brier-loss never exceeds 2 and hence cannot blow up too much.
    Consequently, unlike for the log-loss, a growing distribution shift does not result in an exceedingly large gap between our method and the baselines.

    \clearpage

    \subsubsection{Ablation on the effect of $\alpha$}\label{app:ablation-alpha-risk-aversion-shift}
    Our relative-likelihood-based credal predictor involves the parameter $\alpha$, which determines the least plausible hypothesis --- in terms of relative likelihood --- that is accepted. Consequently, a larger $\alpha$ results in smaller credal sets, and vice versa. To better understand the effect of $\alpha$, and the resulting interaction with our \ac{cvar}-minimax rule, on the performance in risk-averse prediction, we ablate it, considering the values $\alpha \in \{0.0, 0.2, 0.4, 0.6, 0.8, 0.9, 0.95, 1.0\}$. We perform the same experiment as before, using our rule in combination with the credal set --- obtained from the respective $\alpha$-cut --- induced by CreWare, and monitor the $\cvar_\tau$ of the log-loss. \cref{fig:shift_risk_alpha_pathmnist_log_loss,fig:shift_risk_alpha_bloodmnist_log_loss} show the results for the PathMNIST and BloodMNIST datasets, respectively.

    \begin{figure}[h]
        \centering
        \includegraphics[width=0.8\linewidth]{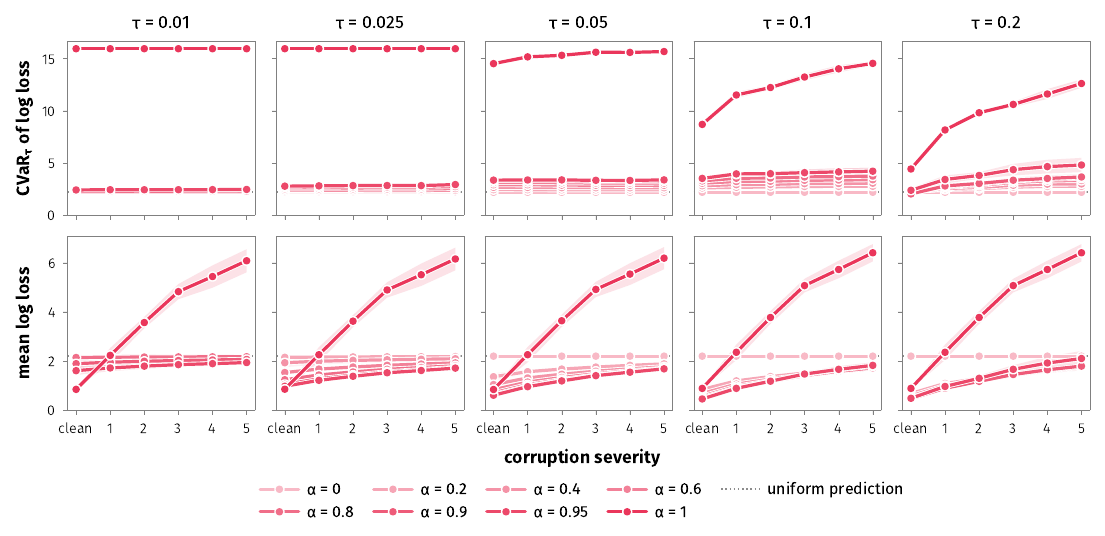}
        \caption{Risk-averse prediction under distribution shift with PathMNIST using different $\alpha$-values.}
        \label{fig:shift_risk_alpha_pathmnist_log_loss}
    \end{figure}
    \begin{figure}[h]
        \centering
        \includegraphics[width=0.8\linewidth]{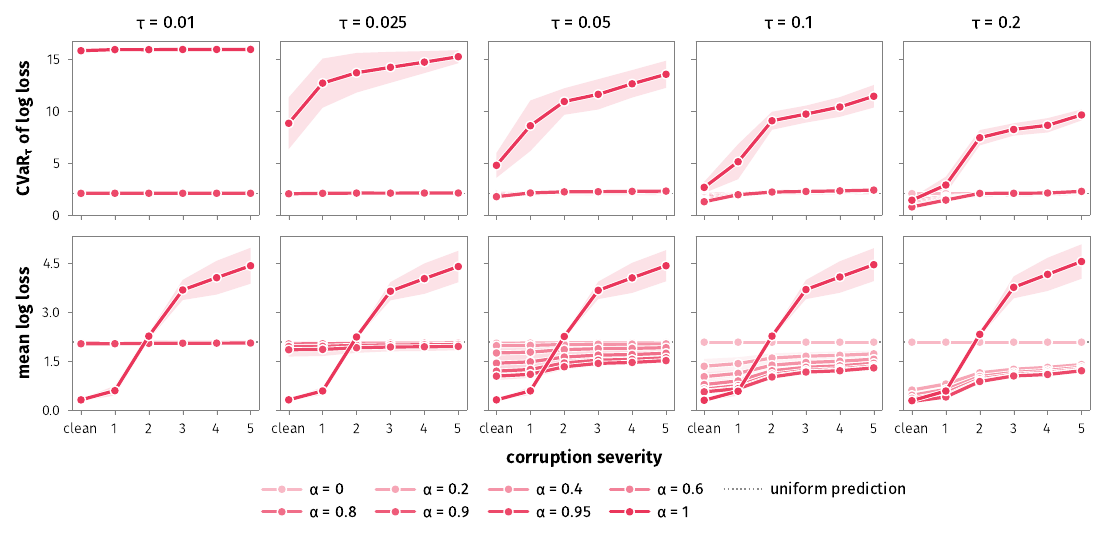}
        \caption{Risk-averse prediction under distribution shift with BloodMNIST using different $\alpha$-values.}
        \label{fig:shift_risk_alpha_bloodmnist_log_loss}
    \end{figure}

    Our method, combining relative-likelihood-based credal prediction and a \ac{cvar}-minimax decision rule, admits a number of special cases, which \cref{fig:shift_risk_alpha_pathmnist_log_loss,fig:shift_risk_alpha_bloodmnist_log_loss} illustrate well. As $\alpha \downarrow 0$, it can be seen that the predictions of the pipeline tend toward the uniform prediction. This trivial prediction strategy performs well in terms of \ac{cvar}, but, obviously, sacrifices much in terms of mean loss. On the contrary, when $\alpha = 1$, the resulting predictions reduce to the precise \ac{cvar}-optimal prediction based on the maximum likelihood estimate, which performs best in terms of mean loss on clean data, while attaining a higher \ac{cvar}. However, as the distribution shift becomes stronger, the $\alpha = 1$ predictions result in increasingly higher losses due to the MLE not being able to express its epistemic uncertainty.

    \clearpage

    \subsection{Cost-Sensitive Decision Making}\label{app:cost-sensitive-results}
    We evaluate the complete pipeline in realistic decision making, where, unlike in standard benchmarks, the cost of wrong decisions are rarely symmetric. In such cost-sensitive settings, the decision maker selects an action from a finite set, evaluated by a cost matrix over the true outcomes. The rule \cref{eq:cvar-distribution-set} carries over directly, with the cost of an action in place of the prediction loss. We simulate a triage scenario where a patient is in one of three heath conditions (healthy, flu, sepsis). The decision maker chooses between sending home, treating for flu, and ICU admission. While appropriate care costs at most $2$ and precautionary interventions $6$--$8$, sending a septic patient home costs $20$. The scenario is difficult for two reasons: some septic patients present with features resembling those of healthy patients, and septic patients are underrepresented in the training data. We compare our pipeline to the maximum likelihood estimate (MLE), SQwash and AdaCVaR, and two ablations by applying our decision rule to the MLE prediction and to credal sets of CreWra. The detailed setup is described in \cref{app:cost-sensitive}.

    \Cref{fig:cost-sensitive-app} shows the results. Our approach attains the lowest tail cost, reducing the \ac{cvar} of the MLE from $10.2$ to $6.7$ at a mean cost of $1.97$ against the MLE's $1.63$. The rule applied to the point estimate or to CreWra sets lowers the tail as well, but at roughly three times the mean cost, since nearly every patient is then admitted to the ICU. The mean cost decomposition in \cref{fig:cost-sensitive-app} shows that the mean cost of the risk-neutral baseline is dominated by catastrophic decisions (e.g., septic patients sent home), something that is minimized by our approach. Our approach acts at prediction time, hence ambiguous cases are hedged by the tail-aware rule, and the underrepresented class by the width of the credal set.

    \begin{figure}
        \centering
        \includegraphics[width=0.9\linewidth]{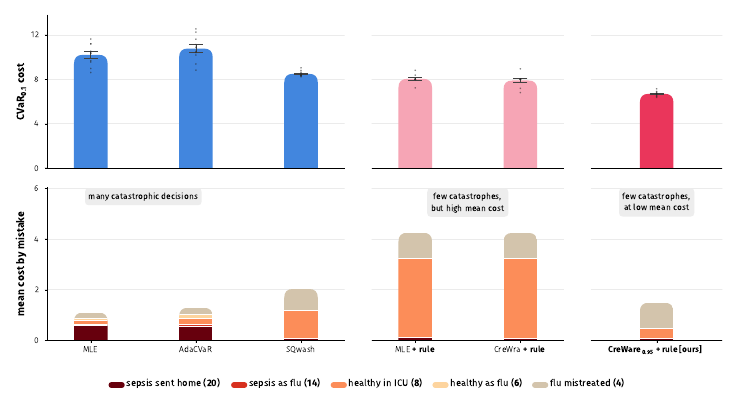}
        \caption{Cost-sensitive decision making. CreWare shows lower \ac{cvar} and avoids catastrophic decisions.}
        \label{fig:cost-sensitive-app}
    \end{figure}

    \Cref{fig:cost_sensitive_lines} varies the training set coverage of the septic cases from $0$ to $120$ cases. The MLE improves only as data accumulate, SQwash avoids catastrophes at any positive coverage but collapses onto the MLE's failure when no sepsis was observed. The point-estimate baselines hedge against worst-case outcomes as soon as some sepsis cases are represented in the data. Our approach is nearly flat throughout (\ac{cvar} $6.9$ to $6.4$): risk-averse training can only act on evidence in the data, whereas the credal set avoids catastrophic decisions also if evidence is missing.

    \Cref{fig:cost_sensitive_bars} contrasts two levels of epistemic uncertainty. With $6$ subgroup sepsis cases (the main-text setting), the ordering of \cref{fig:cost-sensitive-app} holds. With $30$ cases the MLE recovers (\ac{cvar} $7.4$), the ablations no longer retreat to blanket admission, and the arms converge (CreWra+rule $6.36$, ours $6.49$). The credal advantage is largest where evidence is scarcest and diminishes as evidence accumulates.

    \begin{figure}[h]
        \centering
        \includegraphics[width=0.95\linewidth]{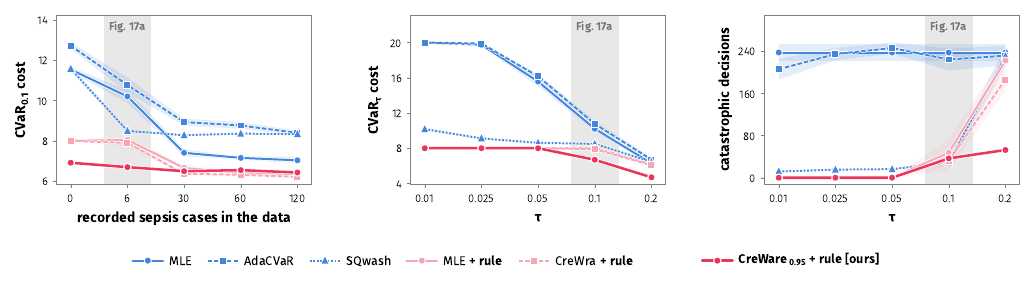}
        \caption{Ablation over the training coverage of the underrepresented subgroup ($\tau = 0.1$, $10$ seeds): our approach is nearly invariant to the missing data, while SQwash's protection collapses at $0\%$. Shaded bands show the standard deviation of the mean over ten seeds.}
        \label{fig:cost_sensitive_lines}
    \end{figure}

    \begin{figure}[h]
        \centering
        \begin{minipage}{0.49\linewidth}
            \centering
            \includegraphics[width=\linewidth]{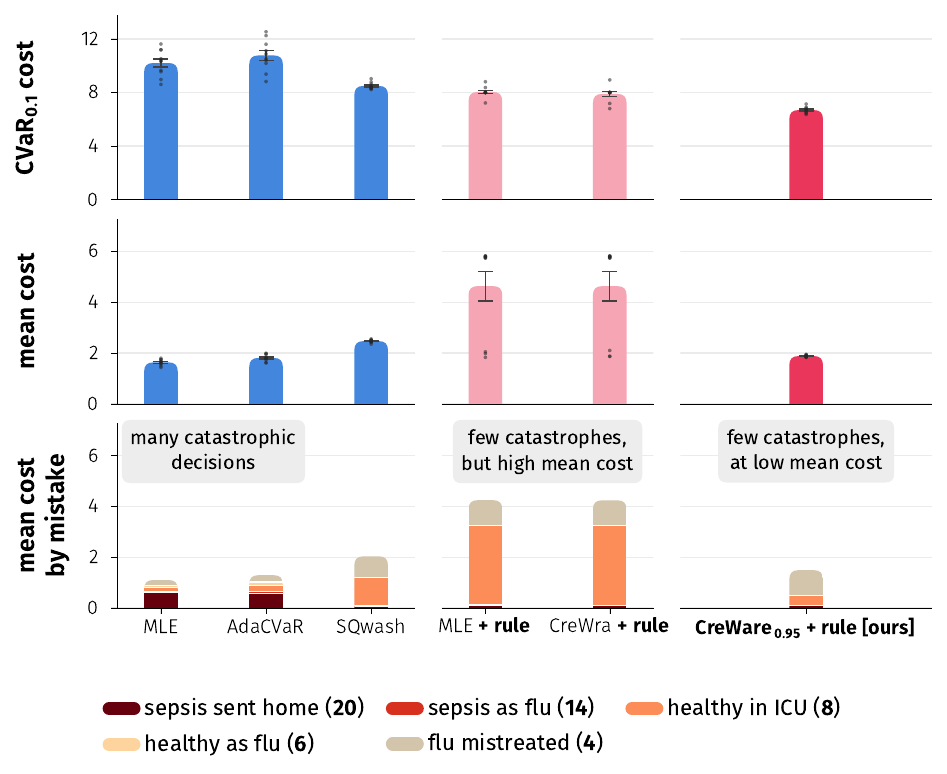}
            \\
            \textbf{(a) high epistemic uncertainty (6 sepsis cases)}
        \end{minipage}
        \hfill
        \begin{minipage}{0.49\linewidth}
            \centering
            \includegraphics[width=\linewidth]{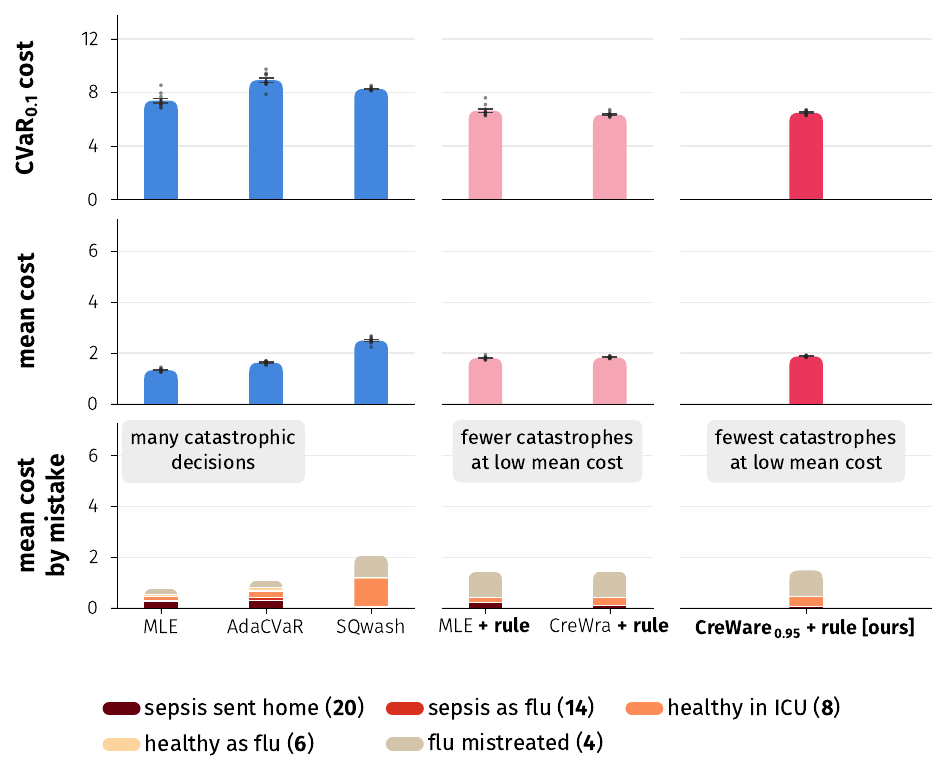}
            \\
            \textbf{(b) less epistemic uncertainty (30 sepsis cases)}
        \end{minipage}
        \caption{The comparison at high ($6$ subgroup sepsis cases, a) and lower ($30$ cases, b) epistemic uncertainty: the credal advantage narrows as evidence accumulates.}
        \label{fig:cost_sensitive_bars}
    \end{figure}

    \clearpage

    \subsection{Risk-Averse Reinforcement Learning}\label{app:rl-results}
    \Cref{fig:driving2} repeats the experiment of \cref{sec:rl-driving} on a larger map, a $100 \times 120$ grid with $8126$ free cells, an eight-directional car (diagonal moves are allowed unless they cut the corner of an obstacle), and two irregularly shaped obstacles with a pass between them. The pass is sealed by a hazard belt twelve rows deep ($329$ hazard cells) with crash probability $p = 0.012$ per cell, so that the direct route, $91$ steps through twelve hazard cells, crashes with probability $1 - (1-p)^{12} \approx 0.135$, while the shortest safe detour, around the left obstacle, takes $136$ steps. With $r_s = -1$, $r_c = -40$, and $r_g = +300$, the shortcut attains an expected return of $169.4$ against $164$ for the detour. Safe cells are observed $n = 10^{4}$ times, which keeps the worst-case crash probability of the detour at $0.025$. The hazard observations are $k \in \{1, 8, 17, 25, 33, 50, 67, 90, 167, 250, 333, 400\}$. Apart from the single observation $k = 1$, they cover approximately the same range of expected observed crashes per cell, $kp \in [0.1, 5]$, as on the small map. The route cap is $s_{\max} = 350$, and the results are over $30$ seeds with $1000$ episodes each. All other settings are as in \cref{app:rl-details}.

    The picture is the same as on the small map. MLE takes the shortcut at every $k$ and crashes in $\approx 13\%$ of the episodes. MLE\plus rule does the same up to $k = 67$, begins to detour at $k = 90$, where it still crosses the hazard in $73\%$ of the seeds, and is safe only from $k = 167$ on, once enough crashes have been observed. Credal\plus rule commits to the detour from the first observation on, with a $0\%$ crash rate at every $k$.
    \begin{figure}[h]
        \centering
        \includegraphics[width=\textwidth]{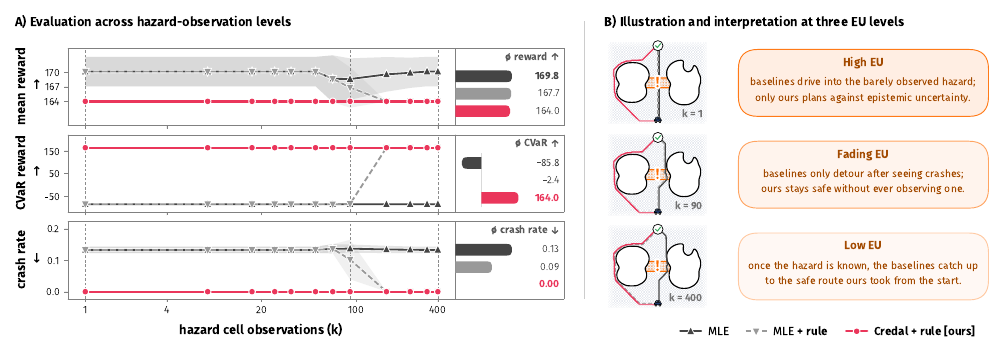}
        \caption{Offline driving task on the larger map: mean reward, CVaR reward, and crash rate vs.\ the number of hazard observations $k$ (left), with route snapshots at three values of $k$ (right) representing levels of epistemic uncertainty. Credal\plus rule (ours) is the only agent safe at every uncertainty level (robust to unobserved danger via the credal set and to known tail risk via CVaR), while MLE always takes the shortcut and the point-estimate risk-averse ablation MLE\plus rule drives into the hazard while it is under-observed.}
        \label{fig:driving2}
    \end{figure}

\end{document}